\PassOptionsToPackage{dvipsnames}{xcolor}
\documentclass[sigconf]{acmart}
\AtBeginDocument{%
  }

\copyrightyear{2026}
\acmYear{2026}
\setcopyright{cc}
\setcctype{by}
\acmConference[KDD 2026] {Proceedings of the 32nd ACM SIGKDD Conference on Knowledge Discovery and Data Mining V.2}{August 9--13, 2026}{Jeju Island, Republic of Korea.}
\acmBooktitle{Proceedings of the 32nd ACM SIGKDD Conference on Knowledge Discovery and Data Mining V.2 (KDD 2026), August 9--13, 2026, Jeju Island, Republic of Korea}
\acmISBN{979-8-4007-2259-2/2026/08}
\acmDOI{10.1145/3770855.3818082}

\usepackage{subcaption}
\usepackage{multirow}
\usepackage{graphicx}
\usepackage{booktabs}
\usepackage{amsthm}

\newcommand{\vecG}{\mathbf{g}}
\newcommand{\vecX}{\mathbf{x}}
\newcommand{\vecY}{\mathbf{y}}
\newcommand{\vecT}{\boldsymbol{\theta}}
\newcommand{\llos}{\mathcal{L}}

\newtheorem{theorem}{Theorem}
\newtheorem{lemma}[theorem]{Lemma}
\newtheorem*{theorem*}{Theorem}
\newtheorem*{lemma*}{Lemma}

\usepackage{wrapfig}
\usepackage{balance}

\makeatletter
\renewcommand\@ACM@checkaffil{%
  \if@ACM@instpresent\else
    \ClassWarningNoLine{\@classname}{No institution present for an affiliation}%
  \fi
  \if@ACM@citypresent\else
    \ClassWarningNoLine{\@classname}{No city present for an affiliation}%
  \fi
  \if@ACM@countrypresent\else
    \ClassWarningNoLine{\@classname}{No country present for an affiliation}%
  \fi
}
\makeatother

\begin{document}

\title{Never Skip a Batch: Dense Learning of Temporal GNNs via Adaptive Pseudo-Supervision}

\author{Alexander Panyshev}
\affiliation{%
  \institution{Risk AI Research Lab}
  \institution{Moscow Independent Research Institute of Artificial Intelligence}
  \city{Moscow}
  \country{Russia}}
\email{panyshev.as@phystech.edu}

\author{Dmitry Vinichenko}
\affiliation{%
  \institution{Risk AI Research Lab}
  \city{Moscow}
  \country{Russia}}
\email{vinichenko.dmitry@gmail.com}

\author{Oleg Travkin}
\affiliation{%
  \institution{Risk AI Research Lab}
  \city{Moscow}
  \country{Russia}}
\email{travkin.o.i@gmail.com}

\author{Roman Alferov}
\affiliation{%
  \institution{Risk AI Research Lab}
  \city{Moscow}
  \country{Russia}}
\email{alferov.r@gmail.com}

\author{Alexey Zaytsev}
\affiliation{%
\institution{Risk AI Research Lab}
}
\affiliation{%
  \institution{Applied AI Institute}
  \city{Moscow}
  \country{Russia}}
\affiliation{%
  \institution{Beijing Institute of Mathematical Sciences and Applications}
  \city{Beijing}
  \country{China}}
\email{likzet@gmail.com}

\renewcommand{\shortauthors}{Alexander Panyshev, Dmitry Vinichenko, Oleg Travkin, Roman Alferov, \& Alexey Zaytsev}

\begin{abstract}
 Temporal graph networks suffer from irregular supervision in real-world dynamic graphs, as most minibatches contain few labeled events.
  The lack of labels leads to high-variance gradient updates and, consequently, slow wall-clock convergence. 
  To constructively reduce sparsity, our Moving-Averaged Labels (MAL) assigns soft pseudo-targets based on past supervised signals using a running label distribution while leaving the loss and the model architecture unchanged. Thus, supervision gaps are replaced with informative signals independent of a temporal graph model and the message passing or memory components used. 
  Theoretical analysis supports our insight that aggregating historical supervision into moving-average targets reduces stochastic gradient variance, yielding faster convergence under mild assumptions. 
  Experimentally, for TGNv2 and DyRepv2 (our modification of DyRep) models, MAL boosts predictive performance, establishing a new SOTA, and improves time-to-accuracy (on average 6\(\times \) faster to reach the top score) for a common suite of Temporal Graph Benchmark datasets.
\end{abstract}




\maketitle

\section{Introduction}

Temporal graphs represent evolving interactions between entities over time.
Their models capture complex patterns, improving predictions in diverse applications spanning from recommendation systems to financial fraud detection~\cite{deng2019learning, song2019session, zhao2019t}.


The current state-of-the-art models for these problems are Temporal Graph Networks (TGNs)~\cite{tgn_icml_grl2020,tjandra2024tgnv2}. 
However, even for TGNs, the training is often impeded by the sparsity of supervision signals. 
Node-level labels (e.g., explicit user preferences or interactions) appear irregularly, leaving many time steps unlabeled.
Current TGN implementations process batches in two modes: full training steps for batches with supervision labels and memory-related state updates for others. 
This creates an efficiency bottleneck as labeled batches constitute less than $2\%$ of interactions for many applied problems~\cite{huang2023temporal}.
As a result, models skip parameter updates for large portions of the data, slowing convergence.

In most temporal node prediction datasets, the target is to predict interaction types, such as a music genre, a subreddit, or a token~\cite{huang2023temporal}, within a specified time frame.
Such interactions can be viewed as a realization of a marked temporal point process, such as a non-homogeneous Poisson process, with a dynamics governed by a slowly changing latent state of an object~\cite{cai2018modeling}.
Although the latent state may experience changes, mostly it remains stable for long intervals, gradually drifting in time~\cite{klenitskiy2024does,li2024sequential}. 
Such slow evolution provides temporal consistency that models can leverage, from temporal-neighbor aggregation~\cite{trivedi2018dyrep} to Temporal Graph Networks with learned memory~\cite{tgn_icml_grl2020}, which achieved state-of-the-art results through learned memory modules.

\begin{figure}
    \centering
    \includegraphics[width=0.95\linewidth]{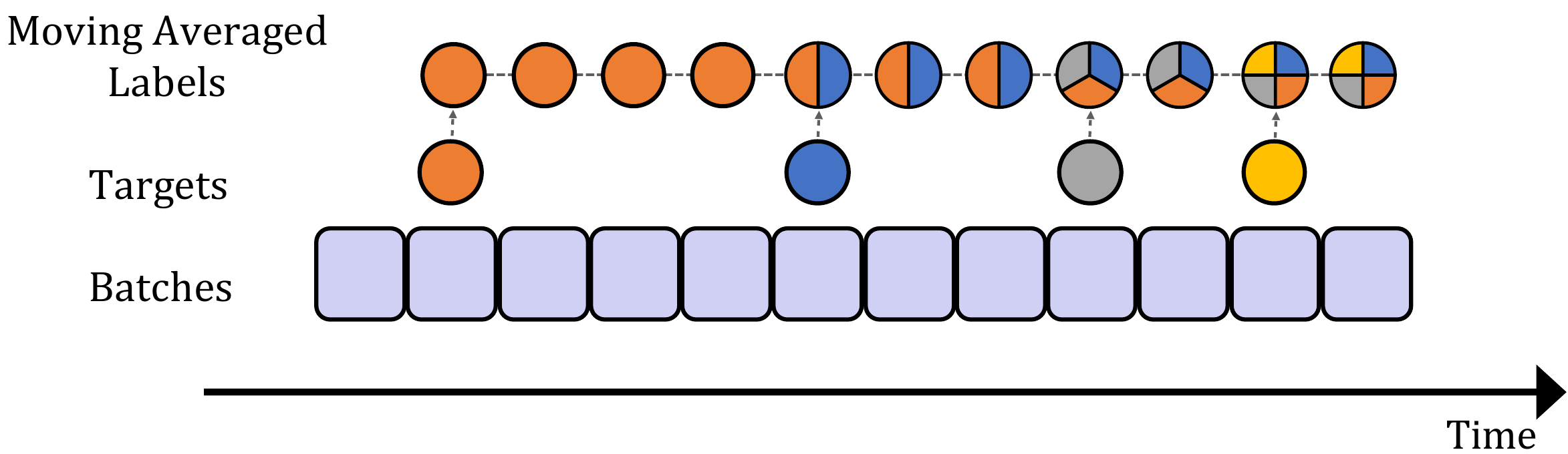}
    \caption{Our training with historical pseudo-labels. We extend beyond the vanilla method's  sparse real ground truth supervision (Targets, the bottom row of circles) by generating pseudo-labels derived from historical label patterns (Historical targets, MAL, the top row of circles), enabling more comprehensive training data utilization.}
    \label{fig:training_pipelines}
\end{figure}

Given the slow dynamics of node preferences, we propose a history-based pseudo-labeling scheme, inspired by the moving average labels (MAL) idea, that turns supervision-free batches into informative training signals, enabling faster training while maintaining model quality for TGNs. 
Formally, the training pseudo-labels for MAL are the exponentially weighted recent targets. 
The MAL operates orthogonally to model architecture choices and introduces no additional parameters. By augmenting the standard cross-entropy loss with pseudo-labeled batches, we transform previously idle computation into productive training steps. 
Our work shifts the paradigm from "train only when supervised" to "always train, intelligently extrapolate," offering a principled solution to gradient sparsity in dynamic graph learning.  
We also compare MAL with an alternative pseudo-label generation strategy, Persistent Forecast, which uses only the most recently observed label.


Comprehensive experiments on the Temporal Graph Benchmark (TGB)~\cite{huang2023temporal} validate our approach. Our experiments demonstrate consistent improvements in model quality across all four benchmark datasets, with NDCG@10 gains ranging from 3\% to 17\% compared to vanilla training. Simultaneously, our method dramatically accelerates the convergence, reducing time-to-target performance on average by 6\(\times\) across datasets.

To summarize, the main claims of this paper are: 
\begin{itemize}
    \item \textbf{History-based pseudo-labeling approach.} We propose a pseudo-labeling based on Exponential History Moving Averaging (MAL) suitable for temporal graphs, reducing label sparsity.
    \item \textbf{Proof of better convergence given MAL.} Under the constant user preferences assumption, we prove faster convergence for history-based aggregation of labels with the theoretical increase of speed $\min(h, k)$,
    where $h$ is the aggregated history length and $k$ is the number of possible interactions associated with a user node.
    \item \textbf{TGN architectures equipped with MAL.} 
    We implemented MAL for training both SOTA TGNv2~\cite{tjandra2024tgnv2} and a better-performing variant of the DyRep~\cite{trivedi2018dyrep} model, which we call DyRepv2. Experimental evidence shows consistent improvements with MAL: training time reduces on average by 6 times, and target quality metrics improve compared to vanilla training and alternative Persistent Forecast aggregation strategy across both architectures and all considered benchmark datasets from TGB.
    \item \textbf{Comparison with alternative training strategies and ablation studies.}
    We benchmark MAL against two related ways of training, label smoothing and a self-supervised (SSL) next-edge-prediction and show that MAL performs better. We further ablate MAL's additional hyperparameters, the noise scale $\gamma$ and the aggregation window $w$, and probe its robustness under controlled label sparsity and temporal label shuffling, confirming that the gains stem from temporal consistency rather than from regularization alone.
\end{itemize}

\section{Related work}
\label{sec:related_work}

Temporal graphs have attracted significant attention due to their importance in real-world problems involving gradual developments in architectures and loss functions.
The paths taken reflected peculiarities of such data, while mostly focusing on expanding ideas from graph neural network models.

\paragraph{Temporal graph node classification}

Early approaches like TGAT~\cite{xu2020tgat} used temporal attention for neighbor aggregation, while JODIE~\cite{kumar2019predicting} learned coupled user-item embeddings via recurrent updates. The introduction of Temporal Graph Networks (TGNs)~\cite{tgn_icml_grl2020} with memory modules and continuous-time message passing marked a significant advance, enabling state-of-the-art performance on dynamic tasks. TGNv2~\cite{tjandra2024tgnv2} further improved expressivity by encoding node identities, addressing limitations in capturing persistent patterns. Recent work explores transformer architectures~\cite{yu2023towards} and lightweight MLP-based models~\cite{cong2023graphmixer}, balancing accuracy and efficiency.

\paragraph{Sparsity of labels}
Scarcity of available labels appears even in regular graphs with no temporal component~\cite{zhan2021mutual}.
However, introducing pseudo-labels provides limited benefits for most existing methods, as methods suffer from information redundancy and noise in pseudo-labels~\cite{li2023informative}.
Moreover, these methods introduce a new training stage that uses pseudo-labels generated by a trained model, increasing overall training time and the number of parameters to select from.
Label propagation through nodes is another approach~\cite{zhu2005semi}.
While it enables a better introduction of connections between nodes, more advanced methods are still required to integrate them into graph convolutional networks~\cite{wang2020unifying}.

\paragraph{Generation of pseudo-labels}

Pseudo-labeling techniques have been widely adopted to address label scarcity. Semi-supervised approaches~\cite{lee2013pseudo, xie2020self} generate labels via self-training but suffer from confirmation bias in low-supervision regimes~\cite{arazo2020pseudo}. Temporal knowledge graph (TKG) methods~\cite{han2023tkgf} interpolate entity distributions or missing facts but focus on triple completion, not node-level affinity prediction.
Our key innovation lies in exploiting intrinsic temporal consistency through lightweight aggregation (e.g., moving averages), inspired by time series forecasting~\cite{hyndman2018forecasting}. Unlike model-dependent pseudo-labeling~\cite{li2023informative}, our method requires no auxiliary parameters, using historical interaction patterns to guide training during label-scarce periods. This aligns with streaming learning principles where historical baselines mitigate concept drift~\cite{ma2021dynaboard}.

Previous work has also considered the dynamics of user preferences~\cite{klenitskiy2024does} and their temporal evolution across a broader range of applied problems~\cite{li2024sequential}. 
In both works, the authors conclude that across different datasets, the dynamics principle varies: some problems exhibit near-constant user preferences, which justifies label propagation and smoothing under true-label sparsity.

\paragraph{Convergence for SGD.}
The convergence speed of Stochastic Gradient Descent (SGD) is fundamentally constrained by the variance of stochastic gradients, especially in the presence of label sparsity or noise. Classical results~\cite{shalev2014understanding,Bubeck2015} show that, under natural assumptions, SGD converges at a rate of $ \mathcal{O}(\sigma^2 / T)$, where $T$ is the number of iterations, and $\sigma^2$ is the gradient variance. 
While approaches such as SVRG~\cite{johnson2013accelerating} and Adam~\cite{kingma2015adam} target variance reduction through optimization, relatively little attention has been paid to label-level variance reduction. So, empirically label smoothing and pseudo-labeling methods address supervision sparsity by generating soft targets based on model predictions or historical trends. In a more principled way, different problem statements appear in the literature, with the sequential nature of the data ignored and only a scarce theoretical analysis of convergence speed and possible quality improvements considered.

\paragraph{Research gap}

Current temporal graph methods address sparsity via architectural changes~\cite{tjandra2024tgnv2} or sampling~\cite{cong2023graphmixer}, neglecting pseudo-label-driven acceleration. While CAW-N~\cite{wang2021caw} and similar methods use temporal walks for induction, they introduce sampling overhead and require careful tuning of a pseudo-label generation model. GraphMixer~\cite{cong2023graphmixer} improves efficiency but remains supervision-bound. 
Our work bridges this gap by demonstrating that simple temporal aggregates—requiring no new parameters—transform idle batches into productive training steps, leveraging temporal consistency to achieve faster convergence with potential quality gains, from both empirical and theoretical perspectives.

\section{Method}
\label{sec:method}

\subsection{General pipeline}
Let \( G = (V, E) \) be a temporal graph with vertices $V$ and temporal interaction edges $E$.
Each timestamped edge \( e \) has attributes \( (u, v, t, \mathbf{f}_e) \),
where \( u, v \in V \) are source and destination nodes, \( t \in \mathbb{R}^+ \) is a timestamp, and \( \mathbf{f}_e \in \mathbb{R}^d \) is a vector of edge features of dimension $d$. 

Thus, edges appear and vanish with time, making their prediction a vital problem.
Formally, for a subset of nodes \( V' \subseteq V \), \(|V'| = n \) (e.g., users), we predict time-varying affinity toward other nodes or categories (e.g., items, music genres). 
The target \( \mathbf{y}_t^{(v)} = (y_{t, 1}^{(v)}, \ldots, y_{t, n}^{(v)})\in \Delta^n \) for node \( v \in V' \) is then an \( n \)-dimensional probability vector representing normalized preferences at time $t$ where \(\Delta^n = \{\vecY: \sum y_i = 1, y_i \geq 0 \}\) is an $n$-dimensional simplex (as defined in Section~\ref{sec:sgd_convergence}).
Each element \( y_{t,i}^{(v)} \geq 0 \) quantifies the affinity of \( v \) toward the \( i\)-th category at time \( t \). 
These targets are observed irregularly, with most vectors lacking ground-truth affinities.    

During training, the edges are processed chronologically and divided into sequential batches \( \{B_1, \ldots, B_T\} \), where each batch \( B_t \) contains a fixed number \( N \) of consecutive edges \( \{e_{t_1}, \ldots, e_{t_N}\} \) that correspond to time moment $t$.
As targets for a single temporal edge are rarely available, batches also contain sparse labels. This sparse supervision creates the fundamental challenge we address: labeled batches may represent less than 2\% of interactions.
 
In the Temporal Graph Network framework, batch processing operates in two distinct modes depending on whether supervision targets are present in the batch.

\begin{figure*}[t]
    \centering
    \includegraphics[width=\textwidth]{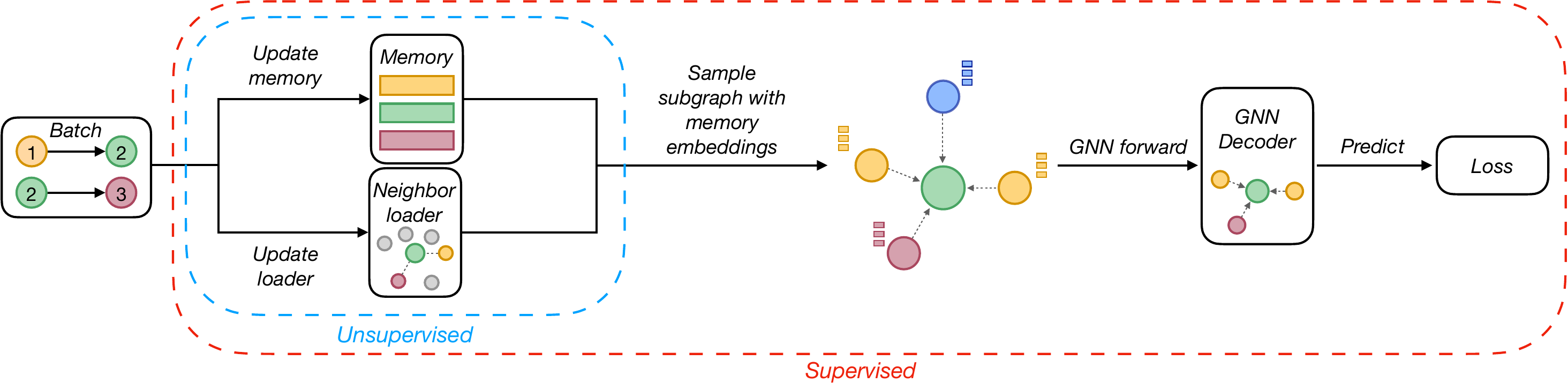}
    \vspace{-20pt}
    \caption{Comparison of batch processing pipelines: The unsupervised pipeline (blue dashed box) performs memory and neighbor loader updates only, while the supervised pipeline (red dashed box) encompasses these steps and extends them with subgraph sampling, GNN processing, and loss computation for model training.}
    \label{fig:batch_with_targets}
\end{figure*}

\paragraph{Unsupervised batch processing.}  
When a batch \( B_t \) arrives without supervision targets, the model performs only memory updates, without any gradient-based training. Temporal edges in \( B_t \) are used to update node memories by aggregating interaction history. The neighbor sampler is also refreshed to incorporate these new interactions for future subgraph construction.  
Since no ground-truth labels are available, no loss is computed, and model parameters remain unchanged during this step.

\paragraph{Supervised batch processing.}  
When \( B_t \) contains supervision targets \( \{\mathbf{y}_t^{(v)}\} \), the model performs a full training step, which includes both memory updates and gradient-based optimization.

First, temporal edges in \( B_t \) that occurred before the target time are used to update node memories. This ensures that predictions are based on historical information only (i.e., no information leakage from the future). The neighbor sampler is also updated accordingly.

Then, for each target node \( v \in V' \), the following steps are performed:

\begin{enumerate}
    \item Temporal subgraph sampling: A subgraph is extracted using the updated sampler, constrained to edges occurring before time \( t \).
    \item Initial embeddings: The memory module provides time-aware node embeddings reflecting their historical state.
    \item Forward pass: A graph neural network (GNN) processes the subgraph to compute context-aware node embeddings.
    \item Prediction and loss: The model generates task-specific predictions, which are compared to the ground-truth labels \( \mathbf{y}_t^{(v)} \). The resulting loss \( \mathcal{L} \) is computed.
    \item Backpropagation: The loss is backpropagated to update: GNN parameters (e.g., message-passing layers, attention mechanisms) and memory-related parameters (e.g., GRU or RNN weights controlling memory updates over time).
\end{enumerate}

The complete workflow for processing batches is illustrated in Figure \ref{fig:batch_with_targets}, showing how the pipeline integrates memory updates with neural network training.

    
    

The model \( f_\theta \) processes edge sequences to update node embeddings \( \mathbf{h}_v^{(t)} \) and predict target labels $\hat{\vecY}_t^{(v)}$, with loss computed over all active nodes:
\begin{eqnarray}\label{eq:lce}
    \mathcal{L}(\vecT) = \frac{1}{T}\sum_{t=1}^T \sum_{v \in \mathcal{V}_t} \mathcal{L}_{\mathrm{CE}} \left(\hat{\vecY}_t^{(v)}, \vecY_t^{(v)} \right) \rightarrow \min_{\vecT},
\end{eqnarray}
where \( \mathcal{L}_{\mathrm{CE}} \) denotes the cross-entropy loss.
Also, formally, we constrain the summands to edges in the current batch, but for simplicity of notation, we will use in the text the loss form presented above throughout the paper.
In our loss function $\tilde{\mathcal{L}}$, we replace $\vecY_t^{(v)}$ with pseudo-labels $\tilde{\vecY}_t^{(v)}$, which are non-zero for almost all nodes and batches, keeping everything else unchanged.

\subsection{Generation of historical pseudo-labels}

For each batch \( B_t \) we compute pseudo-targets \( \tilde{y}_t^{(v)} \) \textit{only for nodes \( v \) participating in \( B_t \)}. The core component \( \tilde{y}_t^{(v)} \) is an aggregate of all historical supervision signals observed for node \( v \) prior to \( t \), 
\begin{eqnarray}\label{eq:ydef}
    \tilde{\vecY}_t^{(v)} = \bar{\vecY}_t^{(v)} + \gamma \cdot \epsilon,
\end{eqnarray}
where \( \bar{\vecY}_t \) comes from aggregating past targets, and \( \epsilon \) is noise generated by first sampling \( \epsilon' \sim \mathcal{U}(-\alpha, \alpha) \) and then subtracting the mean 
to ensure \(\sum_i \epsilon_{i} = 0\). This guarantees that the sum \(\sum_i \tilde{y}_{t,i}^{(v)} = 1\). 
Here, \(\gamma \geq 0\) is a noise scaling factor controlling the magnitude of regularization. This formulation ensures that \(\tilde{y}_{t,i}^{(v)}\) remains a valid probability distribution while reducing gradient variance as proven in Theorem~\ref{th:main}.
We can define \( \bar{\vecY}_t \), derived from \( v \)’s past observations, in different ways.
This paper considers the following aggregation strategies.

\paragraph{Persistent Forecast (PF)} Reuses the most recent observed target for \( v \):
\begin{eqnarray}\label{eq:pf}
    \bar{\vecY}_t^{(v)} = \vecY_{\tau}^{(v)}, \quad \tau = \max\{t' \leq t \,|\, \vecY_{t'}^{(v)} \text{ is available}\}.
\end{eqnarray}

\paragraph{Moving Average (MA)} Updates the pseudo-label incrementally when new targets arrive. Suppose a new target \( \vecY_{t}^{(v)} \) is observed for \( v \) at batch \( t \). The updated pseudo-label becomes:  
\begin{eqnarray}\label{eq:ma}  
    \bar{\vecY}_t^{(v)} = \frac{w - 1}{w} \bar{\vecY}_{t - 1}^{(v)} + \frac{1}{w} \vecY_{t}^{(v)},  
\end{eqnarray}  
with \( w \) being a method hyperparameter. 
Initial \( \bar{\vecY}_0^{(v)} = \vecY_{0}^{(v)} \).

\paragraph{Historical Average (HA)} Aggregates all targets for \( v \) across previous batches:
\begin{eqnarray}\label{eq:ha}
    \bar{\vecY}_t^{(v)} = \frac{1}{L^{(v)}} \sum_{\substack{t' < t}} \vecY_{t'}^{(v)},
\end{eqnarray}
where \( L^{(v)} \) is the number of observed targets for a node \(v\).
This setting corresponds to the HAL case in~Subsection~\ref{sec:hal_convergence}, where the user affinities remain constant.
MA and PF make weaker assumptions about the correlation between the current user preferences and the past observed labels.

\paragraph{Correctness and efficiency of aggregations} HA and MA take into account all past labels, but weight them differently, whereas Persistent Forecast uses the last available value.
Also, both $\bar{\vecY}_t^{(v)}$ and $\tilde{\vecY}_t^{(v)}$ are valid probability distribution. 
For $\bar{\vecY}_t^{(v)}$ in MA, it follows easily from the induction rule; for HA, it also holds because it uses the mean of correct distributions; and for PF, it is evident.
For $\tilde{\vecY}_t^{(v)}$, the introduced normalization procedure leads to the desired effect.
Updating introduced aggregations over timestamps is efficient and can be done iteratively over batches. 
Numerical experiments show small effects of aggregation of the total computational expenses for model training.

\section{Stochastic gradient descent convergence for historical label averaging}
\label{sec:theory}

In this section, we present our main theoretical result on the convergence speed for the pseudo-label generation approach based on Historically Aggregated Labels (HAL).
For MA aggregation, a similar argument applies, but it requires more nuanced technical derivations, so we opt for a more straightforward approach that yields conceptually the same results.

We begin with a preliminary introduction to the convergence of Stochastic Gradient Descent (SGD). 
Then, we continue with the theorem that provides an upper bound on the convergence rate of HAL under natural assumptions, and compare results with a no-aggregation variant.
The complete proofs are postponed to the Appendix~\ref{sec:proofs}.

\subsection{Preliminaries}
\label{sec:sgd_convergence}

Let 
$D = \{(\vecX_i, \vecY_i)\}_{i = 1}^{m}$ be i.i.d. samples drawn from an unknown
distribution $\mathcal{D}$.
$\vecX \in \mathcal{X} \subseteq \mathbb{R}^{d_x}$, $\vecY \in \mathcal{Y} \subseteq \mathbb{R}^{n}$. 
Below, $\mathcal{Y} = \Delta^n = \{\vecY: \sum y_i = 1, y_i \geq 0, 
\mathrm{dim}(\vecY) = n\}$.
For one-hot encoding, a single component is associated with the true label, $y_i = 1$, and all others are zero. 

Given the sample $D$, the empirical risk minimization problem for the parameter vector $\vecT \in \mathbb{R}^d$ is:
\begin{equation}
    \mathcal{L}(\vecT) = 
    \frac{1}{m} \sum_{i=1}^{m} \ell \left(\vecT; (\vecX_i, \vecY_i) \right) \rightarrow \min_{\vecT \in \mathbb{R}^{d}},
\end{equation}
where $\ell :\mathbb{R}^{d} \times \mathcal{X} \times \mathcal{Y} \to \mathbb{R}_{+}$ is a differentiable loss function.

Stochastic Gradient Descent (SGD)~\cite{Robbins1951} updates the parameter vector, starting from the initialization $\vecT_0$, via
\begin{equation}
    \label{eq:sgd_update}
    \vecT_{t + 1} = \vecT_{t} - \alpha_{t} \vecG_t,
    \qquad
    \vecG_t = \frac{1}{B} \sum_{j \in \mathcal{B}_t}
    \nabla \ell \bigl(\vecT_t; (\vecX_j, \vecY_j) \bigr),
\end{equation}
where $t = 0, \ldots, T - 1$ is the SGD iteration number, $\mathcal{B}_t$ is a batch of size $|\mathcal{B}_t| = B$ sampled without replacement.
The estimator $\vecG_t$ is unbiased,
$\mathbb{E}[\vecG_t | \vecT_t] = \nabla \mathcal{L}(\vecT_t)$,
but exhibits non–zero variance and has the following form for i.i.d. data:
\begin{equation}
    \label{eq:grad_var}
    \sigma^2 = 
        \mathbb{E} \|\vecG_t - \nabla \llos(\vecT_t)\|^{2}
    =
    \frac{1}{B}\,\mathbb{E}_{(\vecX, \vecY) \sim \mathcal{D}}
    \|\nabla \ell(\vecT_t; (\vecX, \vecY)) - \nabla \mathcal{L}(\vecT_t)\|^{2}.
\end{equation}

Within our framework, we would replace the original loss function.
Specifically, in $\ell(\vecT_t; (\vecX, \vecY))$ our method would replace $\vecY$ with a pseudo-labels vector $\vecY'$.
As long as $\vecG_t$ is unbiased, 
the theoretical results on convergence speed below hold, as easily follows from~\cite{rakhlin2012making,shalev2014understanding}.

For vanilla SGD, there exists an upper bound for a regret $R_T$ for a diminishing step size $\alpha_t = \frac{1}{\mu t}$: 
\[
    R_T = \mathbb{E}\bigl[ \llos(\vecT_{T}) - \llos(\vecT^{\star}) \bigr],
\]
where $\vecT^{\star}$ denotes the unique minimizer of $\llos$.

\begin{theorem}[adopted from \cite{shamir013}]
\label{th:sgd_upper}
For a $\mu$-strongly convex loss function and an unbiased $\vecG_t$ defined in~\eqref{eq:sgd_update} with variance $\sigma^2$ and batch size $B$ for the step size sequence $\alpha_t = \frac{1}{\lambda t}$, the regret has the upper bound:
\begin{equation}
    \label{eq:sgd_bound}
    R_T
    \le
      \frac{17 \sigma^{2}}{\mu B T} \left( 1 + \log T \right).
\end{equation}
\end{theorem}

This bound is sufficiently tight to describe the real dynamics~\cite{shalev2014understanding}.
Variations of this bound and SGD are also discussed in the literature
\cite{Bottou2018,Bubeck2015}, while the above form would be sufficient for our purposes.

The upper bound
$\sigma^{2}(1 + \log T) / (\mu B T)$ depends
linearly on the gradient--noise variance $\sigma^{2}$ and the inverse of the batch size.
Consequently, we can adjust the convergence speed by minimizing noise variance $\sigma^2$ and maximizing the batch size.
Increasing the batch size $B$ is a common advice for faster convergence, e.g., \cite{you2018imagenet} shows that one can train a ResNet model within 20 minutes using ImageNet with large batches.
Momentum-based approaches for SGD also indirectly increase batch size, improving convergence~\cite{kingma2015adam}.
However, the second component of the variance reduction, $\sigma^2$, related to searching for a lower-variance noise, is often overlooked.

\subsection{Convergence speed for history-average-label SGD}
\label{sec:hal_convergence}
This subsection presents our results on the convergence of SGD with historical pseudo-labels and demonstrates its improvement over vanilla SGD.

Suppose that there are $k$ out of $n$ true labels with a uniform probability of occurrence over them.
Each time, a user selects a label uniformly at random, producing a single label.
We consider an alternative label for a single observation that aggregates past labels to produce $\vecY$ at the current moment, corresponding to our history-average labels.
For this scheme, we calculate the variance of the gradients, $\sigma^2$, obtaining an upper bound for $R_T$ of a standard form.
Specifically, it is decomposed into a product of the parameter values and the labelling-related variance, which is available in closed form.

Canonically, for multilabel classification, $\vecY$ is a one-hot vector with $1$ being at the place of the observed label and $0$ at all other places.
Let us call it \emph{the one-hot label (OH)}.
Aggregating over history and normalising leads to the ground truth vector of the form $\vecY = (\frac{k_1}{h}, \ldots, \frac{k_n}{h})$,
where $h$ is the length of the history of observations and $k_i$ is the number of observations of the $i$-th label within it.
We call it \emph{History Average Labels (HA)}.
Below we consider two options to present $y_i$ --- the $i$-th component of $\vecY$ vector that belongs to the set of the true labels. 

\paragraph{OH case.} 
In this scenario, our random variable $y_i = t_{1}$: 
\[
t_{1} = \eta \xi,
\]
where $\eta$ is the event of observing a specific true label, a Bernoulli random variable $\mathrm{Be}(\frac{1}{k})$, and $\xi$ is another Bernoulli random variable, an indicator of the observation of any of the true labels. It is $\sim \mathrm{Be}(u)$, with $u$ is typically close to $1$.

\paragraph{HA case (ours).} Now let us consider the aggregation of history. 
We assume equal probabilities for each of $k$ correct labels $\frac{1}{k}$ and the history of size $h$. 
Then the presented $y_i = t_{h}$:
\[
t_{h} = \eta_k \xi, 
\]
where $\xi$ is defined above and $\eta_k$ is a component of a multinomial random vector with equal probabilities $\frac{1}{k}$ and the total number of observations $h$, divided by $h$, as we aim to match the event type probability.
As suggested by the notation, $t_1$ is $t_h$ for $h = 1$.
Below, we assume that we form batches from independent observations of HA by considering users separately. 

\begin{lemma}
\label{lemma:HAL_variance}
$\mathbb{E} t_{h} = \frac{u}{k}$ and the variance of $t_{h}$ is~$u \frac{k - 1}{k^2 h} + u (1 - u) \frac{1}{k^2}$.    
\end{lemma}

For the OH and the HA, we have the same unbiased mean value~$\frac{u}{k}$, but the variances differ: the OH's variance is $\sigma^2 = \frac{u}{k} \left(1 - \frac{u}{k}\right)$, and the HA's variance is $\sim \frac{u}{k h} + \frac{u (1 - u)}{k^2}$.
Without compromising tightness, we upper bound the variance of $t_1$ by $\frac{u}{k}$ and of $t_h$ by $\frac{u}{k h} + \frac{u}{k^2}$.  
For a large history length aggregated $h$ or a large number of items in a catalogue $k$, we have lower variance for labels $\sim \frac{1}{k\mathrm{min}(h, k)}$ compared to the order $\frac{1}{k}$ for the first case.
Thus, for $y_{i} = t_h$, we have the variance and the upper bound for it.

Finally, we need to derive \emph{the variance of the gradient} with respect to the parameters, given the variance in labels, derived above. 
Let us consider the last layer before the softmax function.
It has the form $\mathbf{p} = \mathrm{softmax}(C \mathbf{e})$, 
where $\mathbf{e}$ is the embedding vector for the last layer, $C$ is the parameters matrix, and $\mathbf{p}$ is the vector of predicted probabilities for labels.
For an index that corresponds to the correct label \( y \) for the cross-entropy loss function, the gradient is
$\frac{\partial \mathcal{L}}{\partial c_{ij}}  = (p_i - y_i) e_j$.
Thus, the variance of the partial derivative $\mathrm{var}\left( \frac{\partial \mathcal{L}}{\partial c_{ij}} \right) = e_j^2 \mathrm{var}(p_i - y_i) = e_j^2 \mathrm{var}(y_i) = e_j^2 \mathrm{var}(t_h)$.
The main term here is the variance $\mathrm{var}(t_h)$.
For the previous layers, due to the chain rule, we have a similar linear decomposition of the variance of the gradient of the form $\tilde{c} \mathrm{var}(t_h)$, where $\tilde{c}$ is a non-stochastic constant.
Thus, the overall variance of the gradient can be represented in the form $\tilde{c} \mathrm{var}(t_h)$ for some positive constant $\tilde{c} > 0$. 

Plugging in our variance estimates from Lemma~ \ref{lemma:HAL_variance} for $h = 1$ and $h > 1$ into the gradient variance with respect to the parameters, we obtain the convergence speed for OH and HA cases.
\begin{theorem}
\label{th:main}
Consider SGD in settings from Theorem~\ref{th:sgd_upper}.
The following inequalities for the regret hold for a positive constant $c$:    
\begin{itemize}
    \item Under the assumptions of OH, for the regret $R_T$ it holds:
\[
R_T \leq \textcolor{BurntOrange}{\left( 1 - \frac{u}{k}\right)} \frac{u}{k} \frac{c}{\mu B} \frac{1 + \log T}{T} \leq \frac{u}{k} \frac{c}{\mu B} \frac{1 + \log T}{T}.
\]
\item Under the assumptions of HA, for the regret $R_T$ it holds:
\begin{align*}
R_T \leq \textcolor{RoyalBlue}{\left( \frac{k - 1}{k h} + \frac{1 - u}{k} \right)} \frac{u}{k} \frac{c}{\mu B} \frac{1 + \log T}{T}  
\leq \textcolor{RoyalBlue}{\frac{2}{\min(h, k)}} \frac{u}{k} \frac{c}{\mu B} \frac{1 + \log T}{T}.
\end{align*}
\end{itemize}
\end{theorem}

With the rightmost more informal upper bounds for each $R_T$ obtained by upper bounding the colored terms in the left inequalities,
we see that the convergence speed for HA increases by a factor of $\min(h, k)$ when the history is used.
From the theorem above, we can also conclude that for small $k$, an increase of $h$ doesn't affect the convergence speed, as in $\min(h, k)$, $k$ would dominate.
For MA, the role of the history length $h$ is taken by the effective window length~$w$.
Given these results, we theoretically justify a label-aggregation approach for improved convergence of TGN models.

\paragraph{Non-convex case.}
In Theorem~\ref{th:main}, we consider a $\mu$-strongly convex function assumption that often fails to hold.
Similar observations on the crucial role of stochastic gradient variance also hold for the non-convex case and for alternative momentum-based first-order methods~\cite{yan2018unified}, though the specific results would be more cluttered.
Regarding quality, classical statistical learning theory suggests that with lower label variance, we obtain better models with less overfitting~\cite{mohri2018foundations}.


\section{Results}
\label{sec:results}

\subsection{Experiments setups}

\paragraph{Datasets and protocol.}

We use TGB~\cite{huang2023temporal}, a widely adopted benchmark for Dynamic Node Property Prediction.
Our experiments consider all four large-scale dynamic graphs from this benchmark, with varying structural patterns and interaction counts up to $2.5$ million. The datasets span diverse domains: \emph{tgbn-trade} (international agriculture trade, 1986-2016), \emph{tgbn-genre} (user-music interactions), \emph{tgbn-reddit} (user-subreddit activity), and \emph{tgbn-token} (cryptocurrency transactions).  
As the target metric, we use NDCG@10, with higher values indicating better model quality.
Following the TGB~\cite{huang2023temporal} protocol, datasets are split chronologically into training (70\%), validation (15\%), and test (15\%) sets. 




\paragraph{Metrics.}
We evaluate model performance using Normalized Discounted Cumulative Gain at rank 10 (NDCG@10)~\cite{jarvelin2002cumulated}, which measures the ranking quality of the top-10 predicted affinities against ground-truth distributions. For example, in music genre prediction, NDCG@10 quantifies how well the model prioritizes genres a user is likely to engage with, based on their historical listening frequencies. Higher NDCG@10 indicates better model quality.

\paragraph{Dataset properties}

Comprehensive dataset statistics are presented in Table~\ref{tab:datasets}. We evaluate our approach on four datasets spanning diverse scales and domains. Tgbn-trade represents the smallest graph with 255 nodes and 337K edges, while tgbn-genre and tgbn-reddit constitute medium-scale networks with 1.5K-11.7K nodes and 17.9M-27.2M edges. Tgbn-token forms the largest benchmark, comprising 61.7K nodes and 72.9M temporal interactions.

A critical observation emerges regarding the relationship between dataset scale and supervision sparsity. Density-defined as the fraction of batches containing ground-truth labels-exhibits an inverse correlation with graph size. While tgbn-trade maintains 1.30\% label density, this metric degrades to 0.06\% for tgbn-token, meaning labeled batches constitute less than one in a thousand training steps. This extreme supervision scarcity poses a fundamental challenge: the vast majority of parameter updates are foregone in conventional training pipelines, creating a pronounced bottleneck for large-scale temporal graphs where the gradient sparsity problem becomes especially acute.

\begin{table}[h]
\caption{Statistics for used datasets. Density is the number of batches with non-zero labels divided by the total number of batches}
\centering
\label{tab:datasets}
\renewcommand{\arraystretch}{0.75}
\begin{tabular}{ll@{\hspace{8mm}}ll} 
    \toprule
    Dataset & \multicolumn{2}{c}{Number of} & \multicolumn{1}{c}{Density} \\
    \cmidrule(lr){2-3}
    & \multicolumn{1}{c}{Nodes} & \multicolumn{1}{c}{Edges} & \\
    \midrule 
    tgbn-trade & 255 & 337,224 & 1.30\% \\
    tgbn-genre & 1,505 & 17,858,395 & 1.31\% \\
    tgbn-reddit & 11,766 & 27,174,118 & 0.44\% \\
    tgbn-token & 61,756 & 72,936,998 & 0.06\% \\
    \bottomrule
\end{tabular}
\end{table}

\paragraph{Implementation.}

\begin{table}[t]
\centering
\caption{Test NDCG@10 on TGB datasets.}
\label{tab:tgbn-ndcg-test}
\addtolength{\tabcolsep}{-2.5pt}
\renewcommand{\arraystretch}{0.75}
\begin{tabular}{lcccc}
\toprule
Method &  \multicolumn{4}{c}{NDCG@10 Test \( \uparrow \)} \\
\cmidrule(lr){2-5}
& tgbn-trade & tgbn-genre & tgbn-reddit & tgbn-token \\
\midrule
MA (ours)    & 0.729 & \textbf{0.486} & \textbf{0.511} & \textbf{0.344} \\
PF & 0.710 & 0.457 & 0.487 & \underline{0.298} \\
Mean target     & 0.377 & 0.358 & 0.303 & 0.088 \\
\midrule
JODIE           & 0.374 & 0.350 & 0.314 & --    \\
TGAT            & 0.375 & 0.352 & 0.314 & --    \\
CAWN            & 0.374 & --    & --    & -- \\
TCL             & 0.375 & 0.354 & 0.314 & -- \\
GraphMixer      & 0.375 & 0.352 & 0.314 & -- \\
DyGFormer       & 0.388 & 0.365 & 0.316 & -- \\
DyRep           & 0.374 & 0.351 & 0.312 & 0.141 \\
TGN             & 0.374 & 0.367 & 0.315 & 0.169 \\
DyRepv2         & \underline{0.733} & \underline{0.473} & 0.504 & 0.261 \\
TGNv2           & \textbf{0.735} & 0.469 & \underline{0.507} & 0.294 \\
\bottomrule
\end{tabular}
\end{table}


TGNv2~\cite{tjandra2024tgnv2} is the only architecture achieving meaningful performance on TGB; other baselines collapse to trivial solutions. We focus experiments on TGNv2 and additionally modify DyRep~\cite{trivedi2018dyrep} with TGNv2's source-target identification mechanism (DyRepv2) to demonstrate the applicability of label averaging for different architectures. 

We evaluate two aggregation strategies against baselines without pseudo-labeling: Moving Average (MA) and Persistent Forecast (PF) for TGNv2.


\subsection{Main results}

\paragraph{Full-scale performance}
Table~\ref{tab:tgbn-ndcg-test} validates our pseudo-labeling with TGNv2 by comparing it to other architectures.
JODIE~\cite{kumar2019predicting}, TGAT~\cite{xu2020tgat}, CAWN~\cite{wang2021caw}, TCL~\cite{wang2021tcl}, GraphMixer~\cite{cong2023graphmixer}, DyGFormer~\cite{yu2023towards}, DyRep~\cite{trivedi2018dyrep}, and TGN~\cite{tgn_icml_grl2020} achieve NDCG@10 scores statistically indistinguishable from the mean target baseline across datasets, indicating model collapse. 
Only TGNv2 and our proposed DyRepv2 variant consistently surpass this trivial predictor, with TGNv2 and DyRepv2 delivering 2–4$\times$ improvements over collapsed models. 
Our MA further improves TGNv2 scores.

\begin{figure}[t]
\centering
\includegraphics[width=7cm]{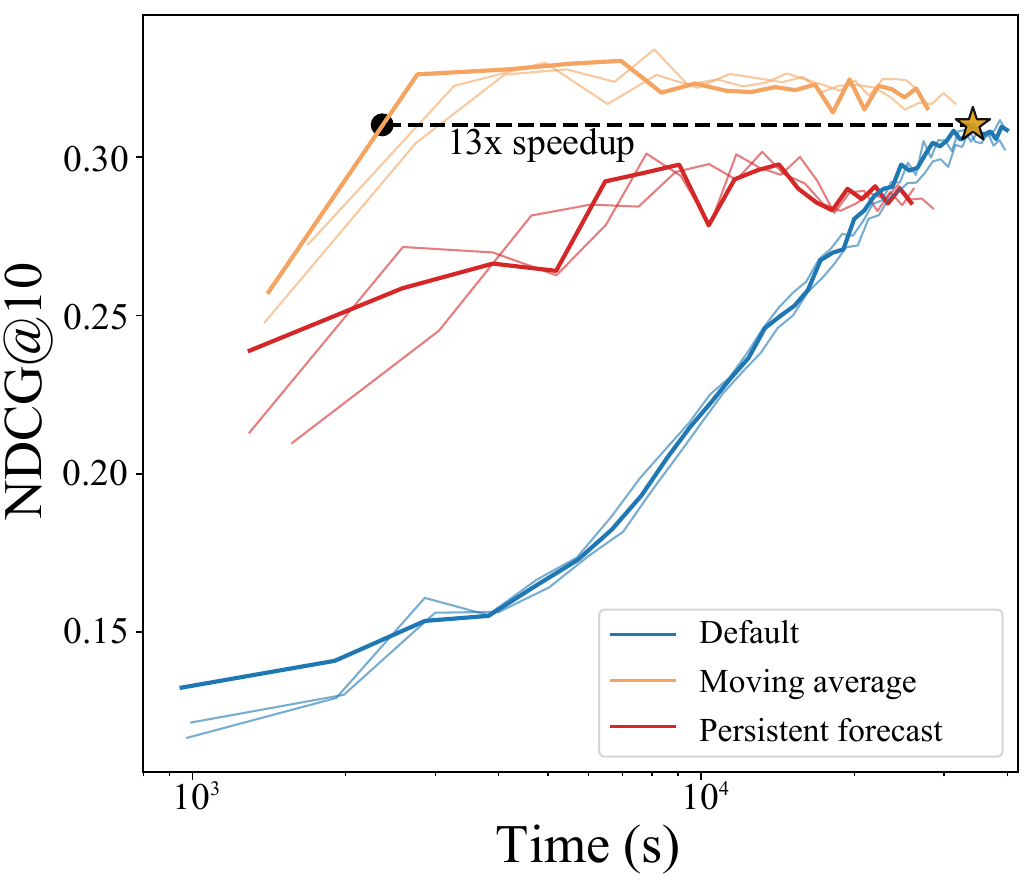}
\vspace{-5pt}
\caption{NDCG@10 progression versus logarithmic training time for different pseudo-label strategies on a 5\% subsample of the tgbn-token dataset using the TGNv2 model. The reduced dataset size enables clearer visualization of training dynamics. The x-axis shows training time in seconds (log scale), while the y-axis displays NDCG@10 on the validation split.}
\label{fig:token_ndcg}
\end{figure}


\paragraph{One-epoch performance}
Table~\ref{tab:main-1epoch} provides a comparison of model performance and efficiency after a single training epoch, and Figure~\ref{fig:comparison_plot} shares the same numbers in a graphical way. 
We evaluate three baseline setups: Default-1 shows the quality achieved after a single epoch of vanilla training, establishing the baseline performance. Default-X demonstrates the quality obtained when training the vanilla model for X epochs, specifically selected so that the total training time roughly matches or slightly exceeds that of our aggregation methods (where X=4 for tgbn-trade and X=2 for tgbn-genre, tgbn-reddit, and tgbn-token). This configuration provides a fair, runtime-controlled comparison that accounts for the additional computational overhead introduced by pseudo-label generation and aggregation. Default represents the performance of the vanilla model trained until convergence, serving as the upper-bound baseline quality. We compare these against our aggregation strategies MA and PF. In addition, due to the similarity between our approach and regularization-based techniques, we include Label
  Smoothing (LS) in the comparison. We further add Self-supervised next-edge prediction (SSL) as a baseline, which is
  a standard choice in low-supervision regimes.

\begin{table*}[t]
  \centering
    \caption{Performance comparison for MA aggregation and Default options. We include Default-X configurations (vanilla model trained for X epochs). For trade, X = 4; for genre, reddit, and token, X = 2. We also report results for Label Smoothing (LS) and Self-supervised next-edge prediction (SSL) baselines. The final column shows the ratio of MA to Default training time. Results averaged over three seeds; standard deviations in Appendix~\ref{sec:add_results}.}
    \vspace{-10pt}
    \label{tab:main-1epoch}
    \addtolength{\tabcolsep}{-2.2pt}
    \renewcommand{\arraystretch}{0.85}
    \begin{tabular}{llccccc ccccc c cccc} 
      \toprule
      Dataset & Model & \multicolumn{7}{c}{NDCG@10 Test \( \uparrow \)} & \multicolumn{7}{c}{Time(s) \( \downarrow \)} & Def / MA  \\
      \cmidrule(lr){3-9} \cmidrule(lr){10-16}
       & & Def-1 & Def-X & Def & LS & SSL & PF & MA & Def-1 & Def-X & Def & LS & SSL & PF & MA & time \\
      \midrule
      \multirow{2}{*}{tgbn-trade} 
      & TGNv2 & 0.386 & 0.449 & \textbf{0.735} & 0.728 & 0.730 & 0.710 & 0.729 & 23 & 84 & 945 & 1022 & 1054 & 68 & 78 & 12.11 \\
      & DyRepv2 & 0.387 & 0.448 & 0.733 & 0.727 & 0.730 & 0.711 & \textbf{0.734} & 26 & 102 & 1043 & 978 & 1103 & 81 & 94 & 11.09 \\
      \midrule
      \multirow{2}{*}{tgbn-genre} 
      & TGNv2 & 0.467 & 0.469 & 0.469 & 0.472 & 0.476 & 0.457 & \textbf{0.486} & 3544 & 7122 & 10232 & 10745 & 14936 & 5182 & 6678 & 1.53 \\
      & DyRepv2 & 0.470 & 0.473 & 0.473 & 0.470 & 0.472 & 0.453 & \textbf{0.482} & 3489 & 7023 & 11932 & 10453 & 15492 & 5049 & 6982 & 1.71 \\
      \midrule
      \multirow{2}{*}{tgbn-reddit} 
      & TGNv2 & 0.454 & 0.470 & 0.507 & 0.506 & 0.506 & 0.487 & \textbf{0.511} & 13564 & 27846 & 257834 & 265295 & 272952 & 26345 & 26543 & 9.71 \\
      & DyRepv2 & 0.453 & 0.472 & 0.504 & 0.505 & \textbf{0.506} & 0.487 & \textbf{0.506} & 14532 & 29593 & 263096 & 258836 & 275836 & 27956 & 27425 & 9.59 \\
      \midrule
      \multirow{2}{*}{tgbn-token} 
      & TGNv2 & 0.167 & 0.197 & 0.294 & 0.292 & 0.290 & 0.298 & \textbf{0.344} & 19245 & 36383 & 172943 & 176264 & 201479 & 36852 & 36960 & 4.68 \\
      & DyRepv2 & 0.164 & 0.190 & 0.261 & 0.264 & 0.259 & 0.285 & \textbf{0.332} & 20149 & 37932 & 178997 & 177467 & 197939 & 37859 & 38485 & 4.65 \\
      \bottomrule
    \end{tabular}
  \hfill
\end{table*}

\begin{figure}[h!]
    \centering
    \includegraphics[width=0.9\linewidth]{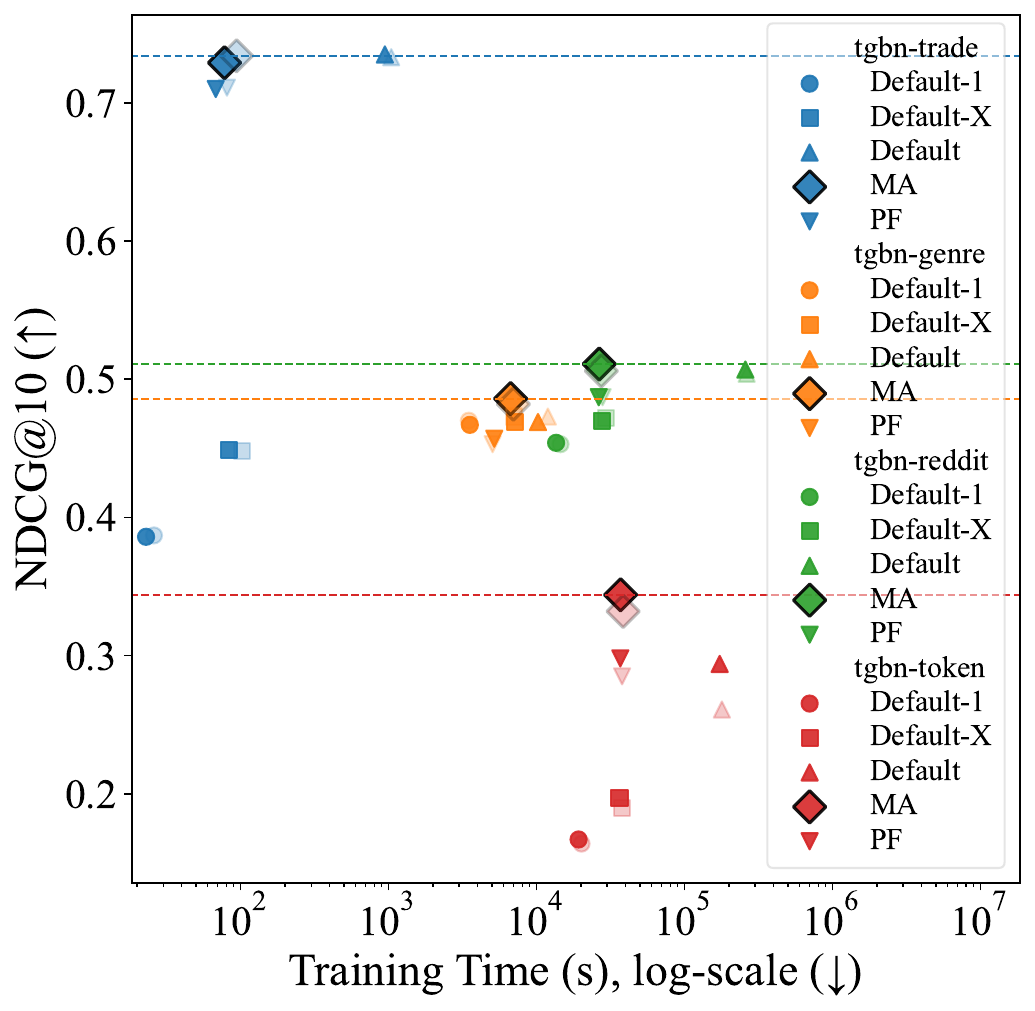}
    \vspace{-10pt}
    \caption{Same metrics plotted, NDCG@10 vs. training time. Darker markers indicate TGNv2; lighter ones indicate DyRepv2. The horizontal lines correspond to the best metrics for each dataset (MA is always the best).}
    \label{fig:comparison_plot}
\end{figure}


The results demonstrate that aggregation strategies not only accelerate convergence but also achieve superior model quality across all datasets. Remarkably, after just a single epoch with aggregation (MA, PF), our methods surpass on three of four datasets the performance of Default models trained until full convergence, while requiring substantially less training time. Aggregation methods achieve higher NDCG@10 scores in a fraction of the time required by vanilla training. Even when compared to Default-X configurations with matched runtime budgets, aggregation-based approaches maintain their performance edge. Moreover, our approach outperforms both Label Smoothing and SSL on most datasets. These findings confirm that pseudo-label aggregation simultaneously improves both predictive accuracy and computational efficiency across diverse temporal graph benchmarks.

\paragraph{Training curve comparison}
Figure~\ref{fig:token_ndcg} shows NDCG@10 progression on a 5\% subsample of tgbn-token, the most challenging dataset due to extreme label sparsity (0.06\% supervision density). We use this reduced subset to clearly illustrate the training dynamics and convergence behavior, as the smaller scale allows for more frequent evaluation and better visualization of the learning trajectory. All aggregation strategies dramatically outperform Default, achieving orders-of-magnitude faster convergence. MA rapidly adapts to recent interaction patterns, while PF shows conservative gains due to its reliance on persistent forecasts in the highly dynamic cryptocurrency transaction setting. The average 6\( \times \) speedup achieved by MA to reach baseline performance highlights the substantial efficiency gains enabled by pseudo-label aggregation.

\subsection{Additional studies}

\paragraph{Scalability across dataset sizes.}

\begin{figure}[t!]
    \centering
    \begin{subfigure}[t]{0.45\linewidth}
        \centering
        \includegraphics[width=\linewidth]{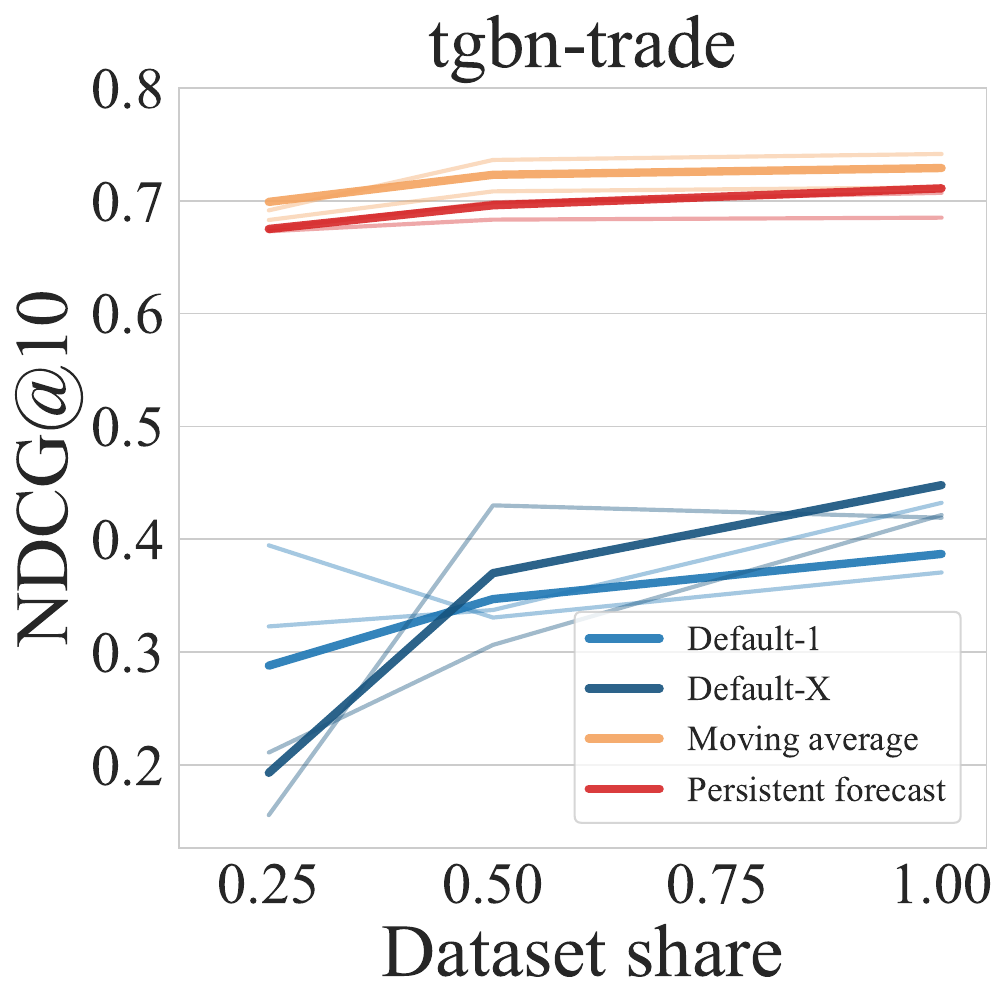}
        \label{fig:window_trade}
    \end{subfigure}
    \hfill
    \begin{subfigure}[t]{0.45\linewidth}
        \centering
        \includegraphics[width=\linewidth]{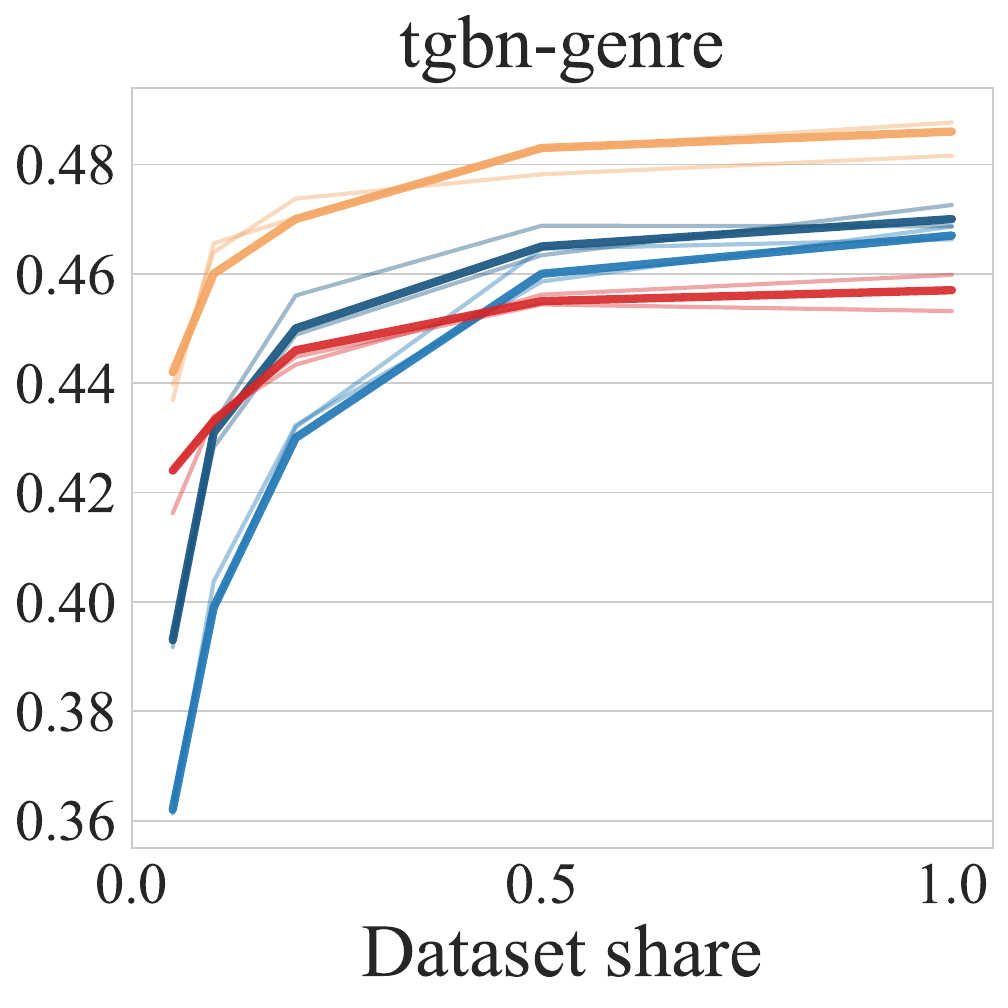}
        \label{fig:window_genre}
    \end{subfigure}
    \vfill
    \vspace{-10pt}
    \begin{subfigure}[t]{0.45\linewidth}
        \centering
        \includegraphics[width=\linewidth]{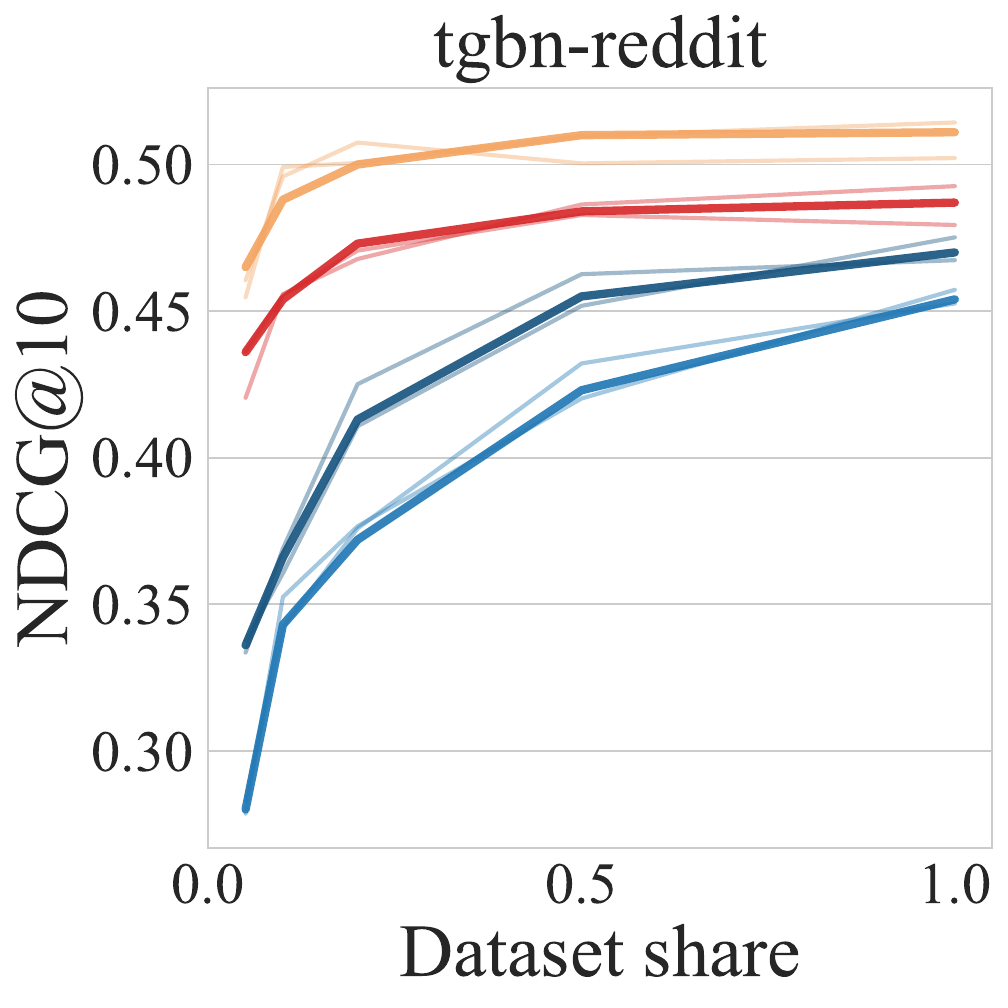}
        \label{fig:window_reddit}
    \end{subfigure}
    \hfill
    \begin{subfigure}[t]{0.45\linewidth}
        \centering
        \includegraphics[width=\linewidth]{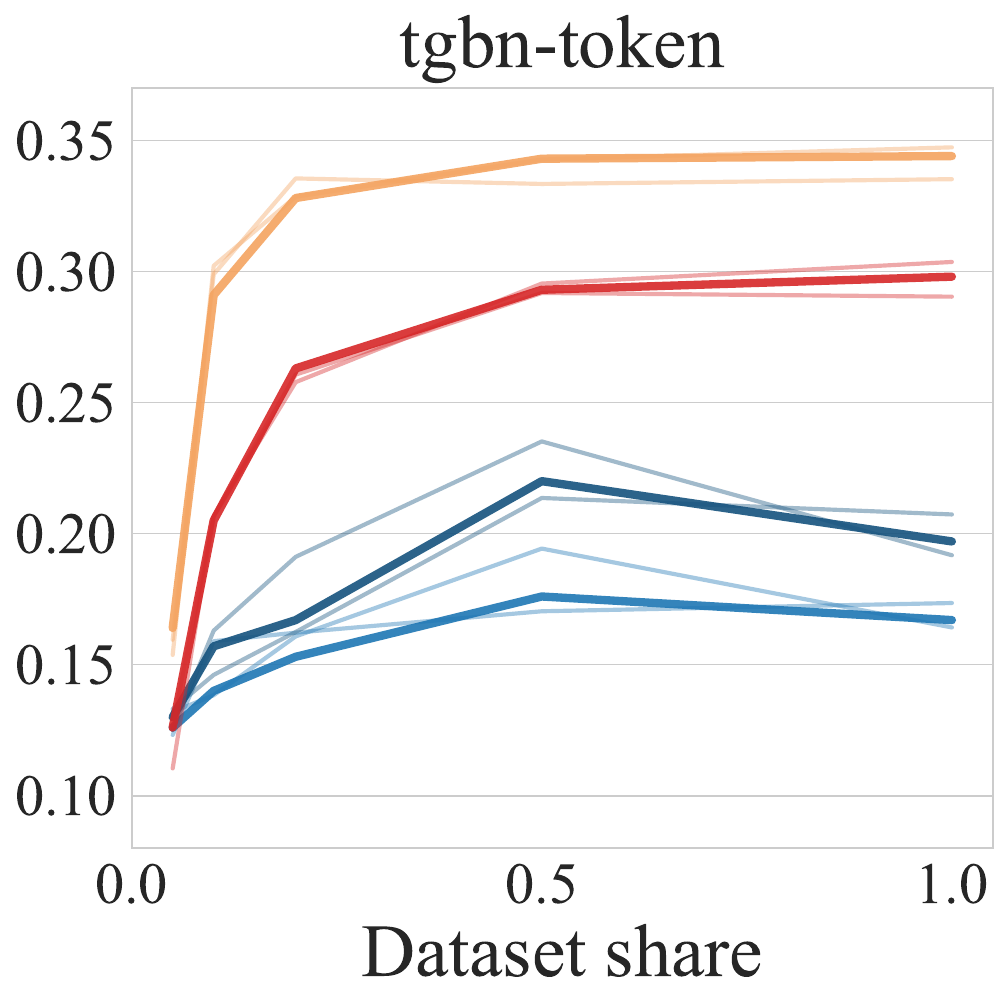}
        \label{fig:window_token}
    \end{subfigure}
    \vspace{-20pt}
    \caption{NDCG@10 versus dataset size after one epoch of training using the TGNv2 model. For greater consistency, we also include Default-X configurations, which show vanilla model training for X epochs. For trade X = 4, for genre, reddit, and token X = 2. Lighter curves correspond to different seeds, while darker ones reflect the averaged curve behaviour.}
    \label{fig:dataset_size}
\end{figure}

Figure~\ref{fig:dataset_size} evaluates robustness to dataset scale by training on 5-100\% subsamples. Moving Average consistently achieves the highest NDCG@10 across all sizes and datasets, with gains over Default ranging from 10-50\% depending on the dataset. Persistent Forecast provides intermediate improvements but consistently trails MA by 2-8\%. Importantly, the relative performance advantage of aggregation methods remains stable across dataset sizes—MA outperforms baselines equally well on 25\% subsamples and full datasets. This demonstrates that the method's efficacy stems from better exploitation of sparse supervision through temporal consistency, rather than from absolute data scale. Even on heavily subsampled tgbn-token (10\% data), MA achieves 0.30 NDCG@10 compared to 0.15 for Default-1, confirming that historical label aggregation effectively addresses supervision sparsity independent of dataset size.

\paragraph{Dependence of Moving Average on Window Size}


\begin{figure}[t!]
    \centering
    \begin{subfigure}[t]{0.45\linewidth}
        \centering
        \includegraphics[width=\linewidth]{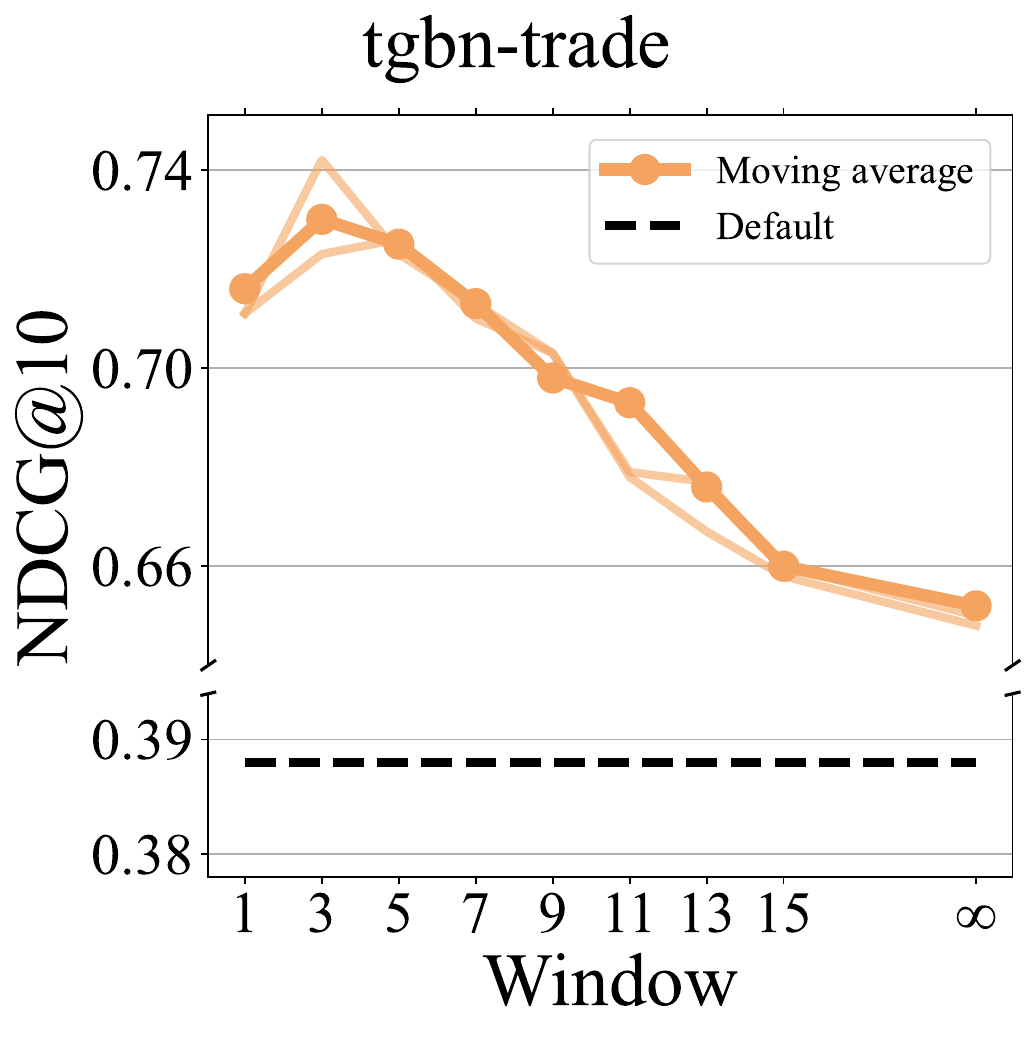}
        \label{fig:window_trade}
    \end{subfigure}
    \hfill
    \begin{subfigure}[t]{0.45\linewidth}
        \centering
        \includegraphics[width=\linewidth]{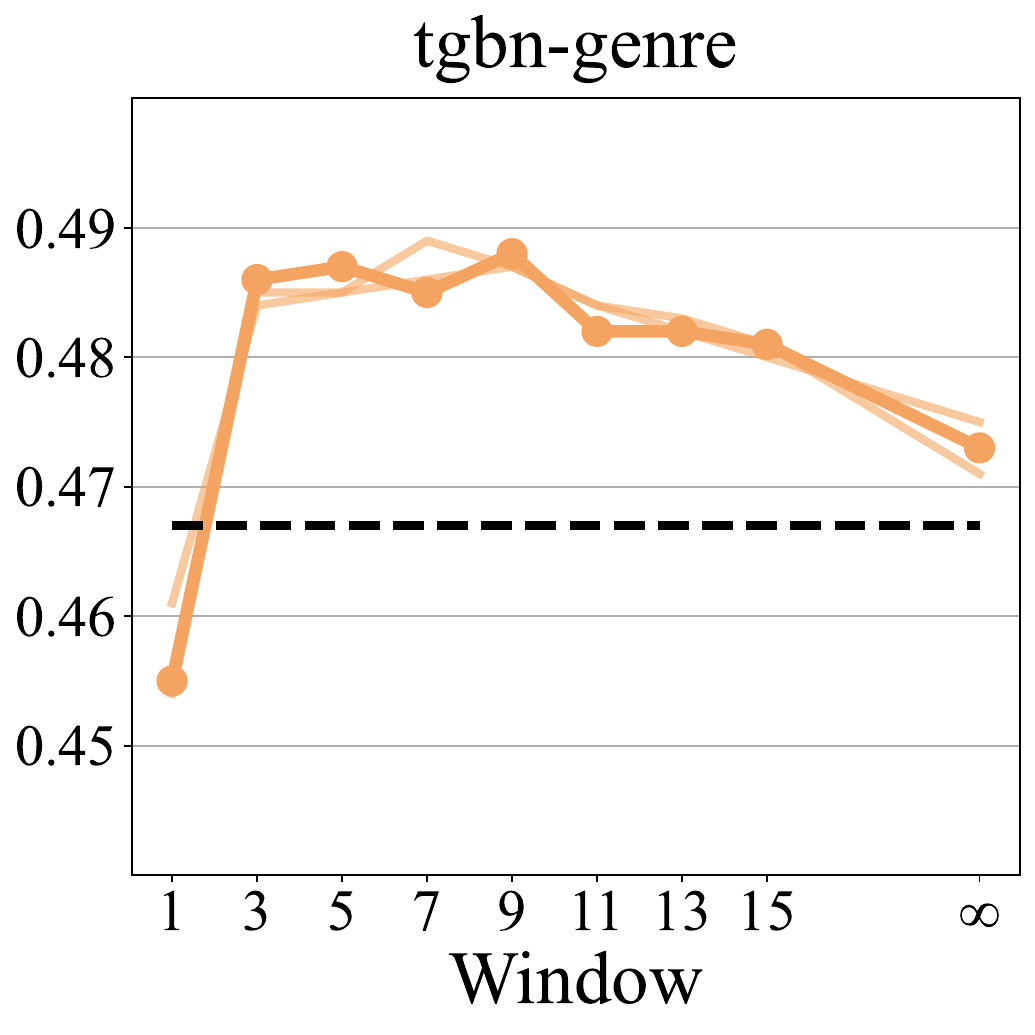}
        \label{fig:window_genre}
    \end{subfigure}
    \vfill
    \vspace{-10pt}
    \begin{subfigure}[t]{0.45\linewidth}
        \centering
        \includegraphics[width=\linewidth]{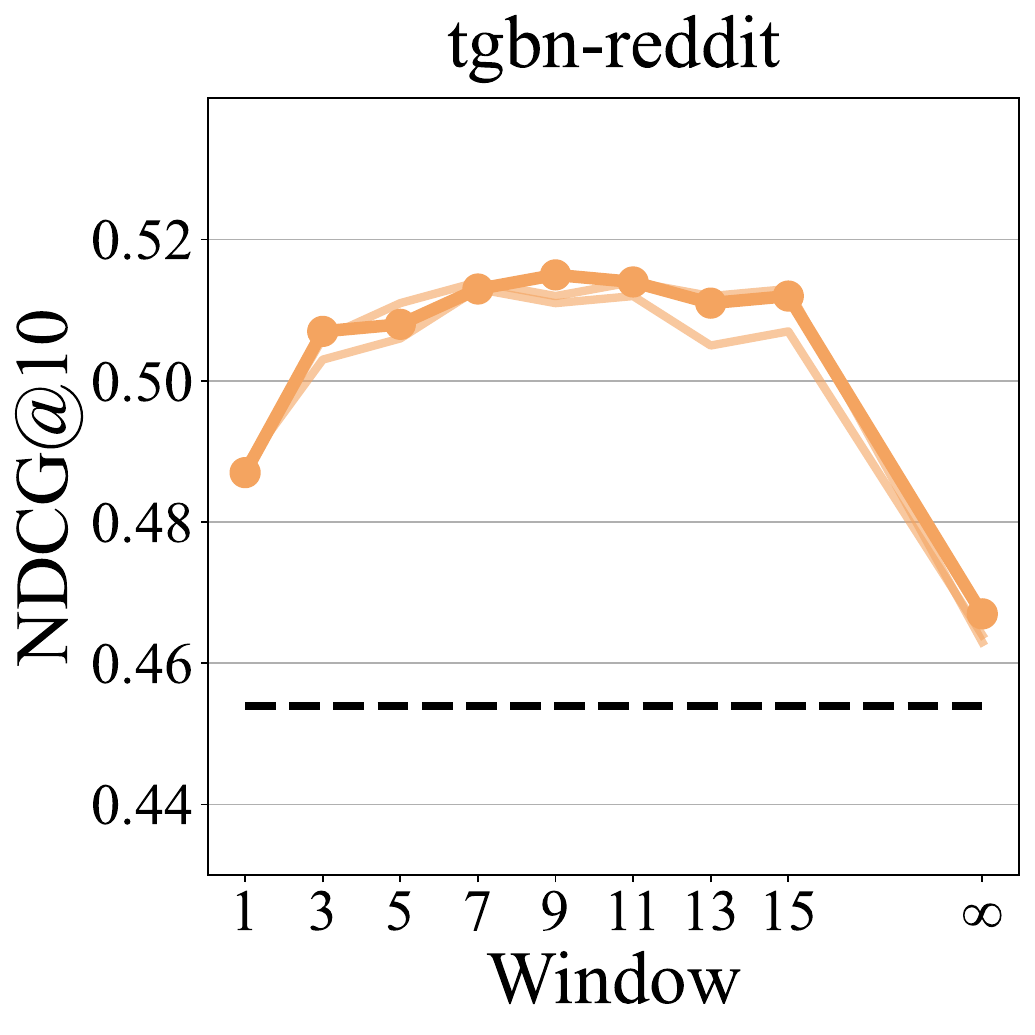}
        \label{fig:window_reddit}
    \end{subfigure}
    \hfill
    \begin{subfigure}[t]{0.45\linewidth}
        \centering
        \includegraphics[width=\linewidth]{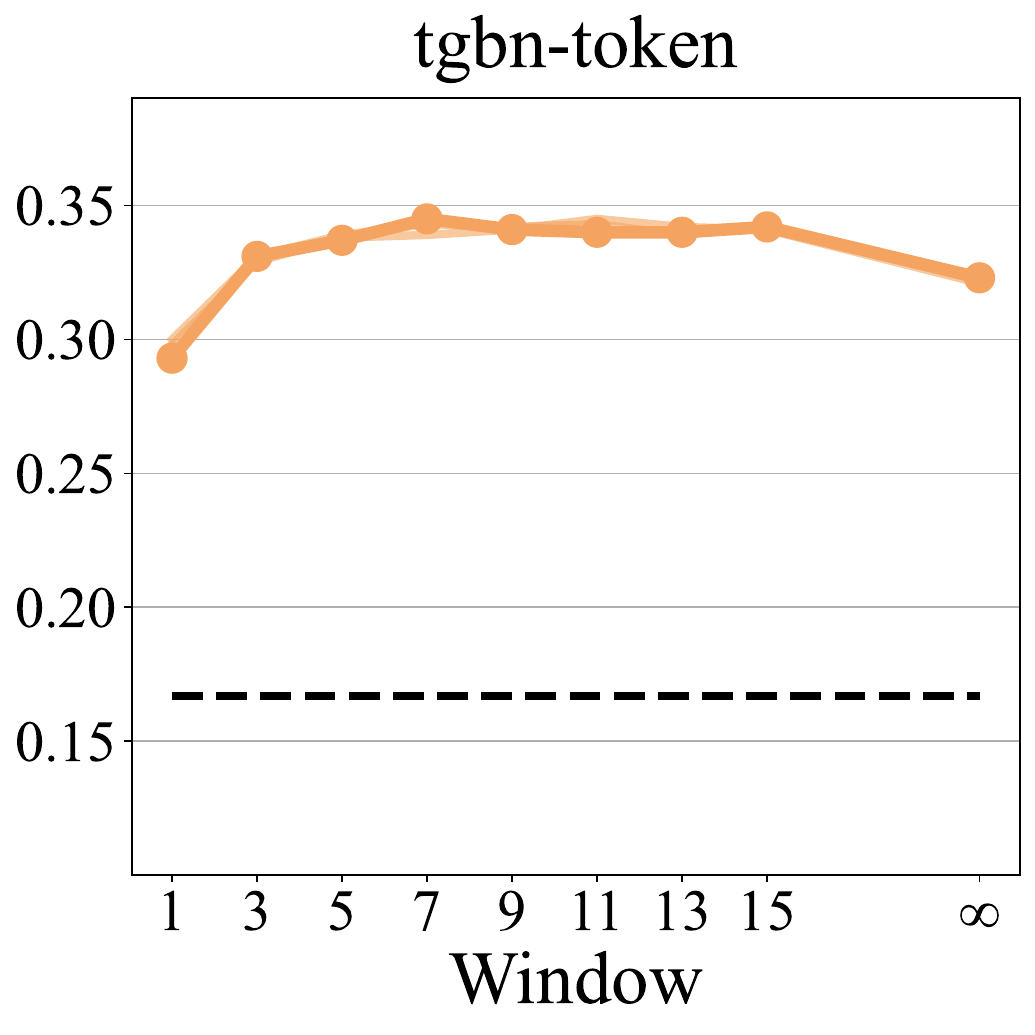}
        \label{fig:window_token}
    \end{subfigure}
    \vspace{-20pt}
    \caption{NDCG@10 versus Moving Average window size after one epoch of training using the TGNv2 model. Dashed lines indicate default TGNv2 performance without pseudo-labels.}
    \label{fig:window}
\end{figure}


The Moving Average (MA) aggregation strategy balances temporal responsiveness and stability through its window size \( w \). Small windows prioritize recent interactions, enabling rapid adaptation but amplifying noise, while large windows emphasize historical trends, stabilizing predictions but potentially delaying response to shifts. Notably, \( w=1 \) corresponds to the Persistent Forecast (PF) strategy, which simply reuses the most recent observed label. At the other extreme, we also evaluate \( w = \infty \) that roughly aggregates all past targets with equal weight.

Figure~\ref{fig:window} shows NDCG@10 versus window size across datasets after one epoch of training. For all considered window sizes and datasets, MA outperforms the Default approach. Performance initially improves with window size as pseudo-labels integrate sufficient context, then plateaus or declines as over-smoothing occurs.

Optimal windows vary significantly by dataset, reflecting differences in temporal dynamics and inherent robustness to aggregation. Trade networks (tgbn-trade) favor short windows (\( w=3 \)), suggesting volatile, rapidly changing interaction patterns. Music genre preferences (tgbn-genre) achieve peak performance across a broader range (\( w \in [3, 9] \)), indicating more stable user tastes that benefit from moderate historical context. User-subreddit activity (tgbn-reddit) performs best with \( w=9 \), reflecting gradual preference evolution in community engagement. Cryptocurrency transactions (tgbn-token) show optimal performance for \( w \in [3, 15] \), demonstrating notable robustness to window size selection.
Overall, selection of window size $7$ can be recommended as the default setting, as it shows suboptimal performance for all cases and allows to improve results over no-aggregation baselines.




\paragraph{Noise factor}

\begin{figure}[t!]
    \centering
    \begin{subfigure}[t]{0.45\linewidth}
        \centering
        \includegraphics[width=\linewidth]{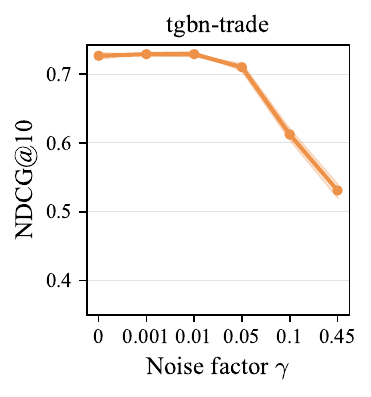}
        \label{fig:noise_trade}
    \end{subfigure}
    \hfill
    \begin{subfigure}[t]{0.45\linewidth}
        \centering
        \includegraphics[width=\linewidth]{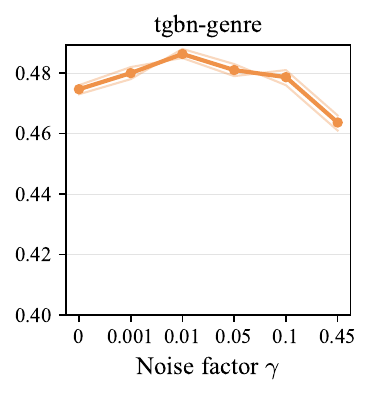}
        \label{fig:noise_genre}
    \end{subfigure}
    \vfill
    \vspace{-10pt}
    \begin{subfigure}[t]{0.45\linewidth}
        \centering
        \includegraphics[width=\linewidth]{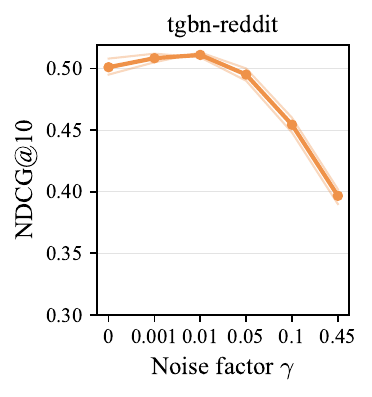}
        \label{fig:noise_reddit}
    \end{subfigure}
    \hfill
    \begin{subfigure}[t]{0.45\linewidth}
        \centering
        \includegraphics[width=\linewidth]{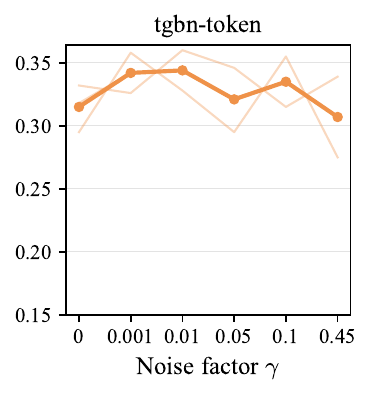}
        \label{fig:noise_token}
    \end{subfigure}
    \vspace{-20pt}
    \caption{NDCG@10 versus noise factor after one epoch of training using the TGNv2 model.}
    \label{fig:noise}
\end{figure}

As defined in Eq.~\eqref{eq:ydef}, the pseudo-label is perturbed by an additive term $\gamma \cdot \epsilon$, where $\gamma \geq 0$ is a scaling factor and $\epsilon$ is a zero-mean random vector that keeps $\tilde{\vecY}_t^{(v)}$ on the probability simplex. Figure~\ref{fig:noise} demonstrates that the method maintains stability across varying $\gamma$, with performance degrading gracefully as noise increases rather than suffering catastrophic failure.

\paragraph{Target shuffle}

Table~\ref{tab:shuffle} demonstrates a comprehensive target shuffling ablation study. When target chronology is randomized, all aggregation methods experience substantial performance degradation across datasets, with absolute NDCG@10 reductions ranging from 4\% to 60\%. This validates the fundamental premise that temporal consistency in node preferences is essential for effective pseudo-labeling in dynamic graphs.

\begin{table}[t]
  \caption{NDCG@10 Test under original and shuffled training labels.}
  \vspace{-10pt}
  \label{tab:shuffle}
  \centering
  \renewcommand{\arraystretch}{0.75}
  \begin{tabular}{llccc}
      \toprule
      Dataset & Strategy & Original & Shuffled & Drop (\%) \\
      \midrule
      \multirow{3}{*}{tgbn-trade}  & Default & 0.735 & 0.662 & $-9.9$  \\
                                   & MA      & 0.729 & 0.688 & $-5.6$  \\
                                   & PF      & 0.710 & 0.647 & $-8.9$  \\
      \midrule
      \multirow{3}{*}{tgbn-genre}  & Default & 0.469 & 0.430 & $-8.3$  \\
                                   & MA      & 0.486 & 0.464 & $-4.5$  \\
                                   & PF      & 0.457 & 0.424 & $-7.2$  \\
      \midrule
      \multirow{3}{*}{tgbn-reddit} & Default & 0.507 & 0.346 & $-31.8$ \\
                                   & MA      & 0.511 & 0.378 & $-26.0$ \\
                                   & PF      & 0.487 & 0.326 & $-33.1$ \\
      \midrule
      \multirow{3}{*}{tgbn-token}  & Default & 0.294 & 0.121 & $-58.7$ \\
                                   & MA      & 0.344 & 0.164 & $-52.4$ \\
                                   & PF      & 0.298 & 0.203 & $-31.8$ \\
      \bottomrule
  \end{tabular}
\end{table}

\paragraph{Controlled sparsity study.}

To quantify the robustness of MA under varying levels of supervision, we artificially remove a fraction of training targets. We choose tgbn-genre as it has the highest target density, which makes controlled removal meaningful. We consider two removal schemes: random, which drops a uniform fraction of targets, and bursty, which removes the most recent targets to simulate an abrupt temporal gap between training and evaluation.

Table~\ref{tab:sparsity} reports NDCG@10 across retention levels. The gain from MA increases monotonically with fewer retained targets, confirming that history-based pseudo-supervision uses the remaining signal more efficiently. Also default degrades sharply under the bursty scheme while MA is largely unaffected, indicating that pseudo-labels transfer well across temporal shifts in the label distribution.

\begin{table}[t]
  \caption{Controlled sparsity study on tgbn-genre. NDCG@10 Test under
  random and bursty target removal at varying retention levels.}
  \vspace{-10pt}
  \label{tab:sparsity}
  \centering
  \renewcommand{\arraystretch}{0.75}
  \begin{tabular}{lcccc}
      \toprule
      \multirow{2}{*}{Retain} & \multicolumn{2}{c}{Random} & \multicolumn{2}{c}{Bursty} \\
      \cmidrule(lr){2-3} \cmidrule(lr){4-5}
              & Default & MA & Default & MA \\
      \midrule
      100\%   & 0.473 & \textbf{0.482} & 0.473 & \textbf{0.482} \\
       50\%   & 0.468 & \textbf{0.478} & 0.439 & \textbf{0.469} \\
       25\%   & 0.461 & \textbf{0.474} & 0.407 & \textbf{0.464} \\
       10\%   & 0.444 & \textbf{0.462} & 0.379 & \textbf{0.455} \\
        5\%   & 0.431 & \textbf{0.453} & 0.357 & \textbf{0.454} \\
      \bottomrule
  \end{tabular}
\end{table}

\section{Conclusion}
\label{sec:conclusion}

We addressed a fundamental bottleneck in training Temporal Graph Networks (TGNs): the inefficiency caused by sparse supervision in real-world temporal data. Motivated by the observation that node-level preferences evolve gradually over time, we propose a simple yet effective pseudo-labeling strategy rooted in temporal aggregation of past labels. Our approach—comprising persistent forecasting and moving averages—requires no additional parameters, incurs negligible computational overhead, and is agnostic to the underlying model architecture.

By employing pseudo-labeled batches, we convert otherwise idle training steps into meaningful updates, significantly accelerating convergence. Our theoretical analysis supports these findings, demonstrating that leveraging multiple plausible labels through historical aggregation improves the convergence speed. 


Empirical results on the Temporal Graph Benchmark confirm the practical benefits of our method: on average 6 \( \times \) faster convergence with improved predictive performance when applied to the state-of-the-art TGNv2 model. Our approach not only accelerates training but also enhances model quality, achieving higher NDCG@10 scores compared to vanilla training across all benchmark datasets. These findings highlight the untapped potential of temporal consistency as a supervisory signal, offering a new direction for efficient learning on dynamic graphs.

Future work may explore adaptive aggregation strategies that learn optimal temporal weighting schemes or extend our method to other tasks such as link prediction or temporal graph completion. Ultimately, our approach contributes a lightweight yet powerful tool for overcoming supervision sparsity in temporal graph learning.

\section{Acknowledgments}

The research was supported by the Russian Science Foundation grant No. 25-11-00355

\bibliographystyle{ACM-Reference-Format}
\balance
\bibliography{sample-base}

\clearpage
\appendix

\section{Proofs of the theoretical results}
\label{sec:proofs}
For convenience, we repeat below the statements we aim to prove.

\begin{lemma*}
The expectation of $t_{h}$ is $\frac{u}{k}$ and the variance of $t_{h}$ is~$u \frac{k - 1}{k^2 h} + u (1 - u) \frac{1}{k^2}$.    
\end{lemma*}

\begin{proof}
We now prove Lemma~\ref{lemma:HAL_variance}.

By definition, $t_{h}$:
\[
t_{h} = \eta_k \xi, 
\]
where $\xi$ is a Bernoulli random variable $\operatorname{Be}(u)$ and $\eta_k$ is a component of a multinomial random vector with equal probabilities $\frac{1}{k}$ and the total number of observations $h$, divided by $h$, as we aim to match the event type probability.
Thus, it has the binomial distribution with parameters $\operatorname{Binomial} \left(h, \frac{1}{k}\right)$, divided by $h$.
These two random variables are independent.

We'll derive the mean and the variance for the random variable that is the product of a Bernoulli and a Binomial random variable.
Then, we'll scale the results by the coefficient $n$.

Let \(B \sim \operatorname{Bernoulli}(u)\), so \(B \in \{0 , 1\}\) and  
  \(\mathbb{E}[B] = u,\; \operatorname{Var}(B) = u (1 -u)\).
Let \(X \sim \operatorname{Binomial}(h, q) \), so  \( \mathbb{E}[X] = hq,\; \operatorname{Var}(X) = h q (1 -q)\);  
\(B\) and \(X\) are independent.

Define the product random variable:  
\[
Y = B X.
\]
Now let us derive the mean and the variance for \(Y\).

We start with the mean.
Because \( B \) and \( X \) are independent,
\[
\mathbb{E}[Y] = \mathbb{E}[B]\mathbb{E}[X] = u h q.
\]

The second moment for \(Y\) is also easy to derive.
Since \(B^2 = B\), as it takes only $0$ and $1$ values,
\[
Y^{2} = B^{2} X^{2} = B X^{2}.
\]

Again, using independence,
\[
\mathbb{E}[Y^{2}]
  =\mathbb{E}[B] \mathbb{E}[X^{2}]
  = u \mathbb{E}[X^{2}].
\]

For a binomial variable,
\[
\mathbb{E}[X^{2}]
  =\operatorname{Var}(X) +(\mathbb{E}[X])^{2}
  = h q (1 - q) + h^{2} q^{2}.
\]

Hence
\[
\mathbb{E}[Y^{2}]
  = u \bigl[h q (1 - q) + h^{2} q^{2} \bigr].
\]

Now we are ready to obtain the variance of \(Y\).
\[
\operatorname{Var}(Y)
  = \mathbb{E}[Y^{2}] -\bigl(\mathbb{E}[Y]\bigr)^{2}
  = u \bigl[h q (1 - q) + h^{2} q^{2} \bigr] - u^{2} h^{2} q^{2}.
\]

Simplifying this expression, we get:
\[
\operatorname{Var}(Y)
  = u h q (1 - q) + u (1 - u) h^{2} q^{2}.
\]

Going back to our original notation, we get the desired mean and variance:
\begin{align*}
\mathbb{E}[t_h] &= \frac{1}{h} u h \frac{1}{k} = \frac{u}{k}. \\
\operatorname{Var}(t_h) &= \frac{1}{h^2} \left(u h \frac{1}{k} \frac{k - 1}{k} + u (1 - u) h^2 \frac{1}{k^2} \right) = \\
 &= u \frac{k - 1}{k^2 h} + u (1 - u) \frac{1}{k^2}.
\end{align*}

\end{proof}

\begin{theorem*}
We consider SGD in settings from Theorem~\ref{th:sgd_upper}.
The following inequalities for the regret hold for a positive constant $c$ for the constants defined above:    
\begin{itemize}
    \item Under the assumptions of OHL, for the regret $R_T$ it holds:
\[
R_T \leq \textcolor{BurntOrange}{\left( 1 - \frac{u}{k}\right)} \frac{u}{k} \frac{c}{\mu B} \frac{1 + \log T}{T} \leq \frac{u}{k} \frac{c}{\mu B} \frac{1 + \log T}{T}.
\]
\item Under the assumptions of HAL, for the regret $R_T$ it holds:
\begin{align*}
R_T &\leq \textcolor{RoyalBlue}{\left( \frac{k - 1}{k h} + \frac{1 - u}{k} \right)} \frac{u}{k} \frac{c}{\mu B} \frac{1 + \log T}{T} \leq \\
&\leq \textcolor{RoyalBlue}{\frac{2}{\min(h, k)}} \frac{u}{k} \frac{c}{\mu B} \frac{1 + \log T}{T}.
\end{align*}
\end{itemize}
\end{theorem*}

\begin{proof}
We now prove Theorem~\ref{th:main}.

Settings of Theorem~\ref{th:sgd_upper} hold, so the regret is bounded by:
\[
    R_T
    \le
      \frac{17 \sigma^{2}}{\mu B T} \left( 1 + \log T \right).
\]

The variance $\sigma^2$ following the discussion in the main part of the paper has the form:
$\sigma^2 = c \operatorname{Var}(t_h)$.
Taking $\operatorname{Var}(t_h)$ from Lemma~\ref{lemma:HAL_variance},
we get 
$\operatorname{Var}(\sigma^2) = c \left( u \frac{k - 1}{k^2 h} + u (1 - u) \frac{1}{k^2} \right)$.

Plugging the expression for  $\sigma^2$ into a general equation for regret for an arbitrary $h$ and $h = 1$,
we get the first pair of the desired bounds in the Theorem.

Now, by construction 
\begin{align*}
  1 - \frac{u}{k} &\leq 1, \\  
  \frac{k - 1}{k h} + \frac{1 - u}{k} &< \frac{1}{h} + \frac{1}{k} \leq \frac{2}{\min(h, k)}.
\end{align*}
There is a pair of rightmost upper bounds.
\end{proof}

\section{Additional experiments setups}



\paragraph{Source code.} The code to reproduce our experiments can be found at \href{https://github.com/PanyshevAlex/NSB_light}{https://github.com/PanyshevAlex/NSB\_light}

The used hyperparameters for MA are presented in Table~\ref{tab:hyperparameters_mal}.
They follow existing practices for considered datasets.
The main method parameter, window size, can be selected pretty much the same for all problems.

\begin{table}[h]
\centering
\caption{Hyperparameters for TGB experiments with Moving Average Labels (MA), for both TGNv2 and DyRepv2}
\label{tab:hyperparameters_mal}
\addtolength{\tabcolsep}{-3pt}
\begin{tabular}{lcccc}
\hline
 & tgbn-trade & tgbn-genre & tgbn-reddit & tgbn-token \\
\hline
Learning Rate & 1e-3 & 1e-4 & 1e-4 & 1e-4 \\
Batch Size & 200 & 200 & 200 & 200 \\
$d$ & 784 & 784 & 784 & 1024 \\
\( N \) neighbours & 25 & 30 & 30 & 10 \\
Noise factor $\gamma$ & 0.01 & 0.01 & 0.01 & 0.01 \\
Window $w$ & 3 & 7 & 9 & 7 \\
Grad. accum. & 1 & 50 & 25 & 1 \\
\hline
\end{tabular}
\end{table}

\label{sec:datasets}

\section{Additional results}
\label{sec:add_results}

\paragraph{Detailed results}

Table~\ref{tab:big-main-1epoch} demonstrates comprehensive performance comparisons across four TGB datasets, showing both NDCG@10 scores and training times after a single epoch. 


Across all four datasets and both architectures, MA achieves the best one-epoch NDCG@10, surpassing not only the matched-runtime Default-X configurations but also reaching or exceeding the fully-trained Default baseline. 
\eject
The largest absolute gains appear on tgbn-token. PF tracks MA closely on tgbn-trade but lags on the larger benchmarks, confirming that exponentially weighted history is more informative than a single most-recent target.


\begin{table}[h]
    \caption{Quality and efficiency metrics after 1 training epoch. For more consistent comparison, we also include Default-X options, which correspond to the training of the vanilla TGNv2 for X epochs. Results averaged across 3 seeds.}
    \label{tab:big-main-1epoch}
    \renewcommand{\arraystretch}{0.973}
    \addtolength{\tabcolsep}{-2pt}
    \footnotesize
    \centering
\begin{tabular}{lllccl} 
    \toprule
    Dataset & Model & Strategy & \multicolumn{2}{c}{NDCG@10 \( \uparrow \)} & Time(s) \( \downarrow \) \\
    \cmidrule(lr){4-5}
     & &  & Test & Val &   \\
    \midrule
    \multirow{10}{*}{tgbn-trade} & \multirow{5}{*}{TGNv2} & Default-1 & 0.386 \(\pm\) 0.006 & 0.405 \(\pm\) 0.008 & 23 \(\pm\) 2 \\
     & & Default-X & 0.449 \(\pm\) 0.003 & 0.499 \(\pm\) 0.003 & 84 \(\pm\) 4\\
     & & Default   & 0.735 \(\pm\) 0.004 & 0.807 \(\pm\) 0.004 & 945 \( \pm \) 32 \\
     & & PF & 0.710 \( \pm \) 0.003 & 0.800 \( \pm \) 0.001 & 68 \( \pm \) 6 \\
     & & MA & 0.729 \( \pm \) 0.010 & 0.813 \( \pm \) 0.015 & 78 \( \pm \) 15 \\

    \cmidrule(lr){2-6}
     & \multirow{5}{*}{DyRepv2} & Default-1 & 0.387 \(\pm\) 0.014 & 0.403 \(\pm\) 0.016 & 26 \(\pm\) 1 \\
     && Default-X & 0.448 \(\pm\) 0.001 & 0.479 \(\pm\) 0.002 & 102 \(\pm\) 2\\
     && Default & 0.733 \(\pm\) 0.001 & 0.802 \(\pm\) 0.002 & 1043 \(\pm\) 35\\
     && PF & 0.711 \(\pm\) 0.003 & 0.800 \(\pm\) 0.002 & 81 \(\pm\) 16 \\
     && MA & 0.734 \(\pm\) 0.003 & 0.821 \(\pm\) 0.001 & 94 \(\pm\) 18 \\
     
    \midrule
    \multirow{10}{*}{tgbn-genre} & \multirow{5}{*}{TGNv2} & Default-1 & 0.467 \(\pm\) 0.005 & 0.478 \(\pm\) 0.005 & 3544 \(\pm\) 69 \\
     & & Default-X & 0.469 \(\pm\) 0.002 & 0.480 \(\pm\) 0.003  & 7122 \(\pm\) 75 \\
     & & Default   & 0.469 \(\pm\) 0.002 & 0.481 \(\pm\) 0.003 & 10232 \(\pm\) 84 \\
     & & PF & 0.457 \(\pm\) 0.003 & 0.462 \(\pm\) 0.002 & 5182 \(\pm\) 95 \\
     & & MA & 0.486 \(\pm\) 0.003 & 0.490 \(\pm\) 0.001 & 6678 \(\pm\) 103 \\

    \cmidrule(lr){2-6}
     & \multirow{5}{*}{DyRepv2} & Default-1 & 0.470 \(\pm\) 0.005 & 0.479 \(\pm\) 0.006 & 3489 \(\pm\) 77 \\
     && Default-X & 0.473 \(\pm\) 0.002 & 0.481 \(\pm\) 0.001 & 7023 \(\pm\) 102\\
     && Default & 0.473 \(\pm\) 0.002 & 0.481 \(\pm\) 0.001 & 11932 \(\pm\) 122\\
     && PF & 0.453 \(\pm\) 0.003 & 0.460 \(\pm\) 0.002 & 5049 \(\pm\) 78 \\
     && MA & 0.482 \(\pm\) 0.001 & 0.489 \(\pm\) 0.001 & 6982 \(\pm\) 94 \\

    \midrule
    \multirow{10}{*}{tgbn-reddit} & \multirow{5}{*}{TGNv2} & Default-1 & 0.454 \(\pm\) 0.006 & 0.492 \(\pm\) 0.007 & 13564 \(\pm\) 194 \\
     & & Default-X & 0.470 \(\pm\) 0.003 & 0.508 \(\pm\) 0.002 & 27846 \(\pm\) 254 \\
     & & Default & 0.507 \(\pm\) 0.003 & 0.545 \(\pm\) 0.002 & 257834 \(\pm\) 1032 \\
     & & PF & 0.487 \(\pm\) 0.019 & 0.518 \(\pm\) 0.016 & 26345 \(\pm\) 385 \\
     & & MA & 0.511 \(\pm\) 0.011 & 0.543 \(\pm\) 0.006 & 26543 \(\pm\) 283 \\

    \cmidrule(lr){2-6}
    & \multirow{5}{*}{DyRepv2} & Default-1 & 0.453 \(\pm\) 0.006 & 0.489 \(\pm\) 0.006 & 14532 \(\pm\) 226 \\
     && Default-X & 0.472 \(\pm\) 0.001 & 0.506 \(\pm\) 0.001 & 29593 \(\pm\) 295 \\
     && Default & 0.504 \(\pm\) 0.001 & 0.539 \(\pm\) 0.001 & 263096 \(\pm\) 1295 \\
     && PF & 0.487 \(\pm\) 0.011 & 0.519 \(\pm\) 0.008 & 27956 \(\pm\) 402 \\
     && MA & 0.506 \(\pm\) 0.011 & 0.540 \(\pm\) 0.003 & 27425 \(\pm\) 243 \\

    \midrule
    \multirow{10}{*}{tgbn-token} & \multirow{5}{*}{TGNv2} & Default-1 & 0.167 \(\pm\) 0.004 & 0.192 \(\pm\) 0.009 & 19245 \(\pm\) 359 \\
     & & Default-X & 0.197 \(\pm\) 0.003 & 0.214 \(\pm\) 0.004 & 36383 \(\pm\) 412 \\
     & & Default & 0.294 \(\pm\) 0.003 & 0.321 \(\pm\) 0.004 & 172943 \(\pm\) 1052 \\
     & & PF & 0.298 \(\pm\) 0.016 & 0.333 \(\pm\) 0.017 & 36852 \(\pm\) 644 \\
     & & MA & 0.344 \(\pm\) 0.024 & 0.385 \(\pm\) 0.009 & 36960 \(\pm\) 586 \\

    \cmidrule(lr){2-6}
     & \multirow{5}{*}{DyRepv2} & Default-1 & 0.164 \(\pm\) 0.004 & 0.185 \(\pm\) 0.008 & 20149 \(\pm\) 405 \\
     && Default-X & 0.190 \(\pm\) 0.003 & 0.210 \(\pm\) 0.004 & 37932 \(\pm\) 449 \\
     && Default & 0.261 \(\pm\) 0.003 & 0.298 \(\pm\) 0.004 & 178997 \(\pm\) 1167 \\
     && PF & 0.285 \(\pm\) 0.019 & 0.312 \(\pm\) 0.016 & 37859 \(\pm\) 837 \\
     && MA & 0.332 \(\pm\) 0.012 & 0.376 \(\pm\) 0.012 & 38485 \(\pm\) 783 \\
     
    \bottomrule
\end{tabular}
\end{table}


The DyRepv2 results follow the same pattern as TGNv2, with MA again best on every dataset. We highlight that the two architectures, despite different memory-update mechanisms, benefit from MA to similar relative degrees.

Table~\ref{tab:main} shows the performance and efficiency of the considered approaches with and without aggregation until convergence.


On tgbn-trade and tgbn-genre, Default matches MA's quality only at 10–12x the training cost. On the sparser tgbn-reddit and tgbn-token datasets with the lowest supervision density Default plateaus below MA.

Importantly, the training-time gap widens on these datasets, since Default requires many additional epochs to compensate for sparse gradient signal whereas MA produces useful updates from every batch.

\begin{table}[h]
    \caption{Quality and efficiency metrics until convergence. Time measurement spans from training start until the best validation epoch, including evaluation phases. For tgbn-genre, tgbn-reddit, and tgbn-token, 5\% subsamples of the datasets were used. Results averaged across 3 seeds.}
    \label{tab:main}
    \addtolength{\tabcolsep}{-2pt}
    \footnotesize
    \centering
\begin{tabular}{lllccl} 
    \toprule
    Dataset & Model & Strategy & \multicolumn{2}{c}{NDCG@10 \( \uparrow \)} & Time(s) \( \downarrow \) \\
    \cmidrule(lr){4-5}
     & &  & Test & Val &   \\
    \midrule
    \multirow{8}{*}{tgbn-trade} & \multirow{4}{*}{TGNv2} & Default & 0.735 \(\pm\) 0.004 & 0.807 \(\pm\) 0.004 & 945 \(\pm\) 32 \\
     & & PF & 0.710 \(\pm\) 0.003 & 0.800 \(\pm\) 0.001 & 68 \(\pm\) 6 \\
     & & MA & 0.729 \(\pm\) 0.010 & 0.813 \(\pm\) 0.015 & 78 \(\pm\) 15 \\

    \cmidrule(lr){2-6}
     & \multirow{4}{*}{DyRepv2} & Default & 0.733 \(\pm\) 0.001 & 0.802 \(\pm\) 0.002 & 1043 \(\pm\) 35 \\
     & & PF & 0.711 \(\pm\) 0.003 & 0.800 \(\pm\) 0.002 & 81 \(\pm\) 16 \\
     & & MA & 0.734 \(\pm\) 0.003 & 0.821 \(\pm\) 0.001 & 94 \(\pm\) 18 \\
     
    \midrule
    \multirow{8}{*}{tgbn-genre 5\%} & \multirow{4}{*}{TGNv2} & Default & 0.432 \(\pm\) 0.007 & 0.437 \(\pm\) 0.006 & 1071 \(\pm\) 22 \\
     & & PF & 0.424 \(\pm\) 0.006 & 0.430 \(\pm\) 0.005 & 238 \(\pm\) 15 \\
     & & MA & 0.442 \(\pm\) 0.004 & 0.450 \(\pm\) 0.003 & 264 \(\pm\) 12 \\

    \cmidrule(lr){2-6}
     & \multirow{4}{*}{DyRepv2} & Default & 0.427 \(\pm\) 0.009 & 0.432 \(\pm\) 0.008 & 1158 \(\pm\) 28 \\
     & & PF & 0.426 \(\pm\) 0.007 & 0.431 \(\pm\) 0.006 & 252 \(\pm\) 18 \\
     & & MA & 0.439 \(\pm\) 0.005 & 0.443 \(\pm\) 0.004 & 278 \(\pm\) 14 \\
     
    \midrule
    \multirow{8}{*}{tgbn-reddit 5\%} & \multirow{4}{*}{TGNv2} & Default & 0.461 \(\pm\) 0.011 & 0.496 \(\pm\) 0.009 & 29890 \(\pm\) 485 \\
     & & PF & 0.436 \(\pm\) 0.012 & 0.470 \(\pm\) 0.011 & 17710 \(\pm\) 324 \\
     & & MA & 0.465 \(\pm\) 0.009 & 0.502 \(\pm\) 0.007 & 2547 \(\pm\) 88 \\

    \cmidrule(lr){2-6}
     & \multirow{4}{*}{DyRepv2} & Default & 0.456 \(\pm\) 0.013 & 0.491 \(\pm\) 0.011 & 30567 \(\pm\) 512 \\
     & & PF & 0.434 \(\pm\) 0.014 & 0.472 \(\pm\) 0.013 & 18256 \(\pm\) 348 \\
     & & MA & 0.463 \(\pm\) 0.010 & 0.505 \(\pm\) 0.008 & 2398 \(\pm\) 91 \\

    \midrule
    \multirow{8}{*}{tgbn-token 5\%} & \multirow{4}{*}{TGNv2} & Default & 0.297 \(\pm\) 0.013 & 0.312 \(\pm\) 0.015 & 45580 \(\pm\) 892 \\
     & & PF & 0.269 \(\pm\) 0.014 & 0.307 \(\pm\) 0.012 & 14071 \(\pm\) 287 \\
     & & MA & 0.301 \(\pm\) 0.009 & 0.337 \(\pm\) 0.007 & 8910 \(\pm\) 198 \\
     
    \cmidrule(lr){2-6}
     & \multirow{4}{*}{DyRepv2} & Default & 0.289 \(\pm\) 0.016 & 0.305 \(\pm\) 0.017 & 47234 \(\pm\) 923 \\
     & & PF & 0.276 \(\pm\) 0.015 & 0.312 \(\pm\) 0.013 & 13789 \(\pm\) 295 \\
     & & MA & 0.308 \(\pm\) 0.010 & 0.343 \(\pm\) 0.008 & 8654 \(\pm\) 206 \\
     
    \bottomrule
\end{tabular}
\end{table}

\balance

\end{document}